\documentclass{article}
\usepackage{microtype}
\usepackage{graphicx}
\usepackage{subfigure}
\usepackage{booktabs} 
\usepackage{hyperref}

\usepackage[accepted]{icml2020}
\usepackage{url}            
\usepackage{amsfonts}       
\usepackage{algorithm}
\usepackage{algorithmic}

\usepackage{amsmath}
\usepackage{graphicx}
\usepackage{multirow}
\usepackage{bm}
\usepackage{color}
\usepackage{amsthm}
\usepackage{tabularx}
\usepackage{footnote}
\usepackage{threeparttable}
\usepackage{tablefootnote}
\makesavenoteenv{tabular}
\makesavenoteenv{table}
\usepackage{caption}
\usepackage{fancyhdr}
\usepackage{amssymb}
\usepackage{pifont}
\usepackage{wrapfig}
\usepackage{exscale}
\usepackage{relsize}
\usepackage{xr}
\usepackage{bbm}

\newcommand{\eg}{\emph{e.g.}}
\newcommand{\ie}{\emph{i.e.}}

\newtheorem{thm}{Theorem}

\newtheorem{lemma}{Lemma}

\newtheorem{cor}{Corollary}

\newtheorem{definition}{Definition}

\DeclareMathOperator*{\argmin}{argmin}

\newcommand{\led}[1]{\overset{\text{\ding{#1}}}{\leq}}

\newcommand{\lee}[1]{\overset{\text{\ding{#1}}}{=}}
\newcommand{\ged}[1]{\overset{\text{\ding{#1}}}{\geq}}

\newcommand{\Rs}[1]{{\mathbb{R}^{#1}}}

\newcommand{\rx}{r}

\newcommand{\Smmi}[1]{\mathcal{S}_{#1}}
\newcommand{\Pmi}[1]{P_{#1}}
\newcommand{\Pmti}[1]{\widetilde{P}_{#1}}

\newcommand{\taui}[1]{\tau_{#1}}
\newcommand{\betai}[1]{\beta_{#1}}
\newcommand{\gammai}[1]{\gamma_{#1}}
\newcommand{\vapi}[1]{\varepsilon_{#1}}
\newcommand{\vapis}[1]{\varepsilon'_{#1}}
\newcommand{\elli}[1]{\ell_{#1}}
\newcommand{\rhos}{L}

\newcommand{\Thetai}[1]{\Theta_{#1}}
\newcommand{\xis}{\xi}

\newcommand{\Ls}{L}
\newcommand{\Hmss}{\bar{\bm{H}}}

\newcommand{\Fms}{F_{\mathcal{S}}}
\newcommand{\Fmts}{\widetilde{F}_{\mathcal{S}}}
\newcommand{\Deltai}[1]{\Delta_{#1}}
\newcommand{\rmi}[1]{\bm{r}_{#1}}

\newcommand{\gmi}[1]{\bm{g}_{#1}}

\newcommand{\Hms}{\bm{H}_{\mathcal{S}}}

\newcommand{\Div}{\mathcal{D}}

\newcommand{\umi}[1]{\bm{u}_{#1}}
\newcommand{\xmi}[1]{\bm{x}_{#1}}
\newcommand{\xmti}[1]{\bar{\bm{x}}_{#1}}
\newcommand{\ymi}[1]{\bm{y}_{#1}}
\newcommand{\ymti}[1]{\bar{\bm{y}}_{#1}}
\newcommand{\wmi}[1]{\bm{\theta}_{#1}}

\newcommand{\zmii}[2]{\bm{z}_{#1}^{#2}}

\newcommand{\wm}{\bm{\theta}}
\newcommand{\wms}{\bm{\theta}^*}
\newcommand{\xm}{\bm{x}}
\newcommand{\ym}{\bm{y}}

\newcommand{\Am}{\bm{A}}
\newcommand{\Ami}[1]{\bm{A}_{#1}}
\newcommand{\Bm}{\bm{B}}

\newcommand{\Hm}{\bm{H}}
\newcommand{\Imm}{\bm{I}}

\newcommand{\Qm}{\bm{Q}}

\newcommand{\Zm}{\bm{Z}}

\newcommand{\Qmi}[1]{\bm{Q}_{#1}}

\newcommand{\SSm}{\mathcal{S}}

\newcommand{\HSDAN}{\textsc{HSDMPG}}

\newcommand{\setsettings}{
	\setlength{\floatsep}{6.5pt} 
	\setlength{\textfloatsep}{6.5pt} 
	\setlength{\intextsep}{6.5pt} 
	\setlength{\dbltextfloatsep}{6.0pt} 
	\setlength{\dblfloatsep}{6.0pt}}
\newcommand{\Oc}[1]{\mathcal{O}\left(#1\right)}

\newcommand{\setbetter}{ \vspace{-0.35em}}
\newcommand{\setbetters}{ \vspace{-0.5em}}

\newcommand{\EE}{\mathbb{E}}

\setsettings

\icmltitlerunning{Hybrid Stochastic-Deterministic Minibatch Proximal Gradient}

\begin{document}
	\twocolumn[
	\icmltitle{Hybrid Stochastic-Deterministic Minibatch Proximal Gradient: Less-Than-Single-Pass Optimization with Nearly Optimal Generalization}
	
	\icmlsetsymbol{equal}{*}
	
	\begin{icmlauthorlist}
		\icmlauthor{Pan Zhou}{to}
		\icmlauthor{Xiao-Tong Yuan}{goo}
	\end{icmlauthorlist}

	\icmlaffiliation{goo}{ B-DAT Lab and CICAEET, Nanjing University of Information Science \& Technology, Nanjing, 210044, China}
	\icmlaffiliation{to}{Salesforce Research}
	
	\icmlcorrespondingauthor{Xiao-Tong Yuan}{xtyuan@nuist.edu.cn}

	\icmlkeywords{Machine Learning, ICML}
	
	\vskip 0.3in
	]
	
	
	
	\printAffiliationsAndNotice{}  

	\begin{abstract}
		Stochastic variance-reduced gradient (SVRG) algorithms have  been shown  to work favorably in solving large-scale learning problems. Despite the remarkable success, the stochastic gradient complexity of SVRG-type algorithms usually scales  linearly with data size and thus could still be expensive for huge data. To address this deficiency, we propose a hybrid stochastic-deterministic minibatch proximal gradient (\HSDAN) algorithm for strongly-convex problems that enjoys provably improved data-size-independent complexity guarantees. More precisely, for quadratic loss $F(\wm)$ of $n$ components, we prove that \HSDAN~can attain an $\epsilon$-optimization-error  $\EE[F(\wm)-F(\wms)] \leq \epsilon$ within $\mathcal{O}\Big(\!\frac{\kappa^{1.5}\epsilon^{0.75} \log^{1.5}\!(\frac{1}{\epsilon}) + 1}{\epsilon}\!  \wedge\!  \Big(\!\kappa \sqrt{n} \log^{1.5}\!\!\big(\frac{1}{\epsilon}\big) \!+\!n \log\!\big(\frac{1}{\epsilon}\big)\!\Big)\!\Big)$ 
		 stochastic gradient evaluations, where $\kappa$ is condition number. For generic strongly convex loss functions, we prove a nearly identical complexity bound  though at the cost of slightly increased logarithmic factors. For large-scale learning problems, our complexity bounds are superior to those of the prior state-of-the-art SVRG algorithms with or without dependence on data size. Particularly, in the case of $\epsilon\!=\!\mathcal{O}\big(1/\sqrt{n}\big)$ which is at the order of intrinsic excess error bound of a learning model and thus sufficient for generalization,  the stochastic gradient complexity bounds of \HSDAN~for quadratic and generic loss functions are respectively $\mathcal{O} (n^{0.875}\log^{1.5}(n))$ and $\mathcal{O} (n^{0.875}\log^{2.25}(n))$, which to our best knowledge, for the first time achieve optimal generalization in less than a single pass over data. Extensive numerical results demonstrate the computational advantages of our algorithm over the prior ones.
	\end{abstract}

	\begin{table*}[th!p]
		\begin{threeparttable}[b]
			\caption{Comparison of IFO complexity for  first-order stochastic algorithms on the $\mu$-strongly-convex problem~\eqref{eqn:general} with condition number $\kappa$. The solution $\wm$ with $\epsilon$-optimization-error  is measured by sub-optimality $\EE\left[F(\wm) \!-\!F(\wms)\right]\!\leq \!\epsilon$ with  optimum $F(\wms)$. Here we define a set of constants for quadratic (generic) loss:  $\betai{1} \!=\!1.5~(2.25)$, $\betai{2}\!=\!1~(2)$, $\betai{3}\!=\!3~(4.5)$, $\betai{4}\!=\!1~(2.5)$, $\betai{5}\!=\!1~(1.5)$,  $\gammai{}\!=\!1.5~(2.25)$. 
				These different constants only affects the logarithm factor $\xis=\log\big(\frac{1}{\epsilon}\big)$. 	For brevity, we define $\Thetai{} \!=\!\frac{\kappa^{1.5} \xis^{\gammai{}}}{\epsilon^{0.25}} + \frac{1}{\epsilon}$. The third column summarizes the conditions under which \HSDAN~has lower IFO complexity than the compared   algorithms.
				\vspace{-0.2em}
			}
			\setlength{\tabcolsep}{1pt}
			\renewcommand{\arraystretch}{0.8}
			\label{comparisontable}
			\begin{center}
				{ \footnotesize {
						\begin{tabularx}{\textwidth}{cc|cl|cc}\toprule
							\multirow{2}{*}{}&&\multicolumn{2}{c|}{$\epsilon$-Optimization Error for ERM~\eqref{eqn:general}} & $\frac{1}{\sqrt{n}}$-Optimization  \\
							&&\multicolumn{1}{c}{IFO Complexity } &\multicolumn{1}{c|}{Better Zoom of \HSDAN} & Error for ERM~\eqref{eqn:general} \\
							\midrule

							\multicolumn{2}{c|}{\multirow{2}{*}{SGD}}& \multirow{2}{*}{$\Oc{\frac{1}{\mu \epsilon}}$ } &$\quad$\ding{172} $\mu\!\leq \!1 \& \mu \kappa^{1.5} \epsilon^{0.75}\xi^{\betai{1}}\! \!\leq \!\mathcal{O}{(1)} $ &\multirow{2}{*}{$\Oc{n}$} \\
							& && or   \ding{173}  $\mathcal{O}(n)\!\leq\!\! \frac{1}{\mu\epsilon  \xis^{\betai{2}}} \wedge  \frac{1}{\kappa^2 \mu^2 \epsilon^2  \xis^{\betai{3}}}$ &  &\\
							\midrule
							
							\multicolumn{2}{c|}{ SVRG, SAGA, }& \multirow{2}{*}{$\Oc{(n+\kappa)\log\big(\frac{1}{\epsilon}\big)} $ } &\multirow{2}{*}{$\quad$\ding{172}   $\Thetai{} \xi^{-1}\leq \Oc{n}$} &\multirow{2}{*}{$\Oc{n\log(n)}$ } \\
							\multicolumn{2}{c|}{APSDCA} &&   &  &\\
							\midrule
							
						\multicolumn{2}{c|}{\multirow{2}{*}{APCG}}&\multirow{2}{*}{$\mathcal{O}\big(\frac{n}{\sqrt{\mu}}\log\big(\frac{1}{\epsilon}\big)\big)$}   &  $\quad$\ding{172} $\Thetai{} \mu^{0.5}\xi^{-1} \leq \Oc{n}$  & \multirow{2}{*}{$\Oc{n^{1.25}\log(n)}$} \\
														& && or   \ding{173}  $ \mu^{-1}\kappa^{2}\xi^{\betai{4}} \leq \mathcal{O}(n)$ &  &\\
						\midrule	
							
							\multicolumn{2}{c|}{ SPDC, Catalyst,  }& \multirow{2}{*}{$\Oc{(n+\sqrt{n\kappa})\log\big(\frac{1}{\epsilon}\big)}$} &\multirow{2}{*}{$\quad$\ding{172}  $\Thetai{} \xi^{-1}\! \wedge\!  \Thetai{}^2\xi^{-2} \kappa^{-1} \! \leq\! \Oc{n}$} &\multirow{2}{*}{$\Oc{n\log(n)}$ } \\
							\multicolumn{2}{c|}{Katyusha} &&   &  &\\
							\midrule
							
							\multicolumn{2}{c|}{AMSVRG}&$\mathcal{O}\big(\big(n+\frac{n\kappa}{n+\sqrt{\kappa}}\big)\log\big(\frac{1}{\epsilon}\big)\big)$   & $\quad$\ding{172}  $\Thetai{} \xi^{-1}  \leq \Oc{n} $ & $\Oc{n\log(n)}$  \\
							\midrule
							
							\multicolumn{2}{c|}{\multirow{2}{*}{Varag}} &\multirow{2}{*}{$\mathcal{O}\Big(\!n\log\!\big(n \!\wedge\! \frac{1}{\epsilon}\!\big)\!+\!\!\sqrt{n}\Big(\frac{1}{\epsilon^{0.5}}\! \wedge\! \kappa^{0.5}\!\log\!\big(\frac{1}{\epsilon \kappa}\big)\!\Big)\!\Big)$}   &  $\quad$\ding{172}  $\Thetai{}\log^{-1}(n\wedge \frac{1}{\epsilon}) \leq \Oc{n}$  &  \multirow{2}{*}{$\Oc{n\log(n)}$}\\
							& &&or \ding{173} $ \Thetai{}^2 (\epsilon\! \vee \! \frac{1}{\kappa\!\log^{2}(\frac{1}{\epsilon\kappa})}) \!\leq \! \Oc{n} $ &\\
							\midrule

							\multicolumn{2}{c|}{\multirow{2}{*}{SCSG}} &\multirow{2}{*}{$\Oc{(n\wedge\frac{\kappa}{\epsilon}+\kappa)\log\big(\frac{1}{\epsilon}\big)}$}   &   $\quad$\ding{172} $\Thetai{} \xi^{-1}  \leq \Oc{n} \leq \frac{\kappa}{\epsilon}$  & \multirow{2}{*}{$\Oc{n\log(n)}$}\\
							&& & or \ding{173} $\kappa  \epsilon^{1.5} \!\xis^{\betai{5}} \!\leq\! \mathcal{O}(1)$  $\!\&\!$ $\frac{\kappa}{\epsilon} \!\leq\! \mathcal{O}(n)$   \\
							\midrule
							
							\multirow{3}{*}{\HSDAN}& quadratic & $\mathcal{O}\Big(\!\frac{\kappa^{1.5}\epsilon^{0.75} \log^{1.5}\!(\frac{1}{\epsilon}) + 1}{\epsilon}\!  \wedge\!  \Big(\!\kappa \sqrt{n} \!\log^{1.5}\!\!\big(\frac{1}{\epsilon}\big) \!+\!n \!\log\!\big(\frac{1}{\epsilon}\big)\!\Big)\!\Big)$ & $\qquad$--------------------------------- &  $\mathcal{O}\!\left(n^{0.875}\!\log^{1.5}\! \left(n\right) \right)$ \\
							& generic & $\mathcal{O}\Big(\!\frac{\kappa^{1.5}\epsilon^{0.75}\! \log^{2.25}\!(\frac{1}{\epsilon}) + 1}{\epsilon}\!  \wedge\!  \Big(\!\kappa \sqrt{n}  \log^{2.5}\!\!\big(\frac{1}{\epsilon}\big) \!+\!n\! \log^2\!\!\big(\frac{1}{\epsilon}\big)\!\Big)\!\Big)$& $\qquad$---------------------------------  & $\mathcal{O}\!\left(n^{0.875}\!\log^{2.25}\! \left(n\right) \right)$  \\
							\bottomrule
						\end{tabularx}
				}}
			\end{center}
		\end{threeparttable}
	\end{table*}

	\section{Introduction}\label{introduction}
	We consider the following $\ell_2$-regularized empirical risk minimization (ERM) problem:
	\begin{equation}\label{eqn:general}
	\setlength{\abovedisplayskip}{2pt}
	\setlength{\belowdisplayskip}{2pt}
	\setlength{\abovedisplayshortskip}{2pt}
	\setlength{\belowdisplayshortskip}{2pt}
	\min\nolimits_{\wm \in \Rs{d}} F(\wm):= {\frac{1}{n} \sum\nolimits_{i=1}^n \ell(\wm^\top \xmi{i}, \ymi{i}) + \frac{\mu}{2}\|\wm\|_2^2},
	\end{equation}
	where $\{(\xmi{i},\ymi{i})\}_{i=1}^n$ is a training set; the convex loss function $\ell(\wm^\top \xmi{i}, \ymi{i})$ measures the discrepancy between the linear prediction $\wm^\top \xmi{i}$ and the ground truth $\ymi{i}$; and the regularization term  $\frac{\mu}{2}\|\wm\|_2^2$ aims at enhancing generalization ability of the linear model.
	In the field of statistical learning, the formulation~\eqref{eqn:general} encapsulates a vast body of problems including
	least squares regression, logistic regression and  softmax regression, to name a few. In this work, we focus on developing scalable and autonomous first-order optimization methods to solve this fundamental problem, which has been extensively studied with a bunch of efficient algorithms proposed including gradient descent (GD)~\cite{Cauchy1847}, stochastic GD (SGD)~\cite{robbins1951stochastic}, SDCA~\cite{shalev2012online}, SVRG~\cite{SVRG}, Catalyst~\cite{Catalyst}, SCSG~\cite{lei2017less} and Katyusha~\cite{katyusha}.

	\textbf{Motivation.} Despite the remarkable success of the stochastic gradient methods and their variance-reduced extensions, the stochastic gradient evaluation complexity (which usually dominates the computational cost) of these algorithms tends to scale linearly with data size $n$. Such a linear dependence is not only  expensive when data scale is huge but also problematic in online and life-long learning regimes where samples are coming infinitely. As pointed out in~\cite{lei2017less}, there are situations in which accurate solutions can be obtained with less than a single pass through the data, \eg~for a large-scale dataset with similar and redundant samples.  Therefore, developing data-size-independent  learning algorithms is of special importance in big data era.
	
	Particularly, we are interested in efficiently optimizing problem~\eqref{eqn:general} to its  intrinsic excess error bound which typically scales as $\mathcal{O}(1/\sqrt{n})$. As shown in~\cite{bottou2008tradeoffs}, the excess error, which measures the expected prediction discrepancy between the optimum model and the learnt model over all possible samples and thus reflects the generalization performance of the model, can be decomposed into model approximation error, estimation error and optimization error.
	Among them, the model approximation error measures how closely the selected predication model can approximate the optimal model; the estimation error measures the prediction effects of minimizing the empirical risk instead of the population risk; the optimization error denotes the prediction difference between the exact and approximate solutions of ERM. Therefore, to achieve small  excess error, one should minimize the three terms jointly.
	With optimal choice $\mu=\mathcal{O}(1/\sqrt{n})$ to balance empirical risk and generalization gap, the estimation error is known to be at the order of $\mathcal{O}(1/\sqrt{n})$, which implies the excess error is dominated by $\mathcal{O}(1/\sqrt{n})$~\cite{vapnik2006estimation,shalev2009stochastic,shalev2014understanding}.
	Thus, it is sufficient to optimize the regularized ERM problem~\eqref{eqn:general} to the  optimization error  $\mathcal{O}(1/\sqrt{n})$ to match the optimal excess error without redundant computation.

	\textbf{Overiew of our contribution.} The main contribution of this paper is a novel Hybrid Stochastic-Deterministic Minibatch Proximal Gradient (\HSDAN) algorithm with substantially improved data-size-independent complexity over existing methods. For quadratic problems, the core idea of our method is to recurrently convert the original large-scale ERM problem into a series of minibatch proximal ERM subproblems for efficient minimization and update. Specifically, as a starting point, we uniformly randomly select a minibach $\SSm$ of components of the risk function $F$ to form a stochastic approximation $\Fms$ that will be fixed throughout the algorithm iteration. Next, at each iteration step, we first construct a stochastic surrogate of $F$ by combining the Bregman divergence of $\Fms$ at the current iterate and a first-order hybrid stochastic-deterministic approximation of $F$; and then we invoke existing variance-reduced algorithms, such as SVRG, to minimize this surrogate subproblem to desired optimization error. For quadratic loss, we can provably establish sharper bounds of incremental first order oracle (IFO, see Definition~\ref{def:IFO}) for such a hybrid stochastic-deterministic minibatch proximal update procedure in large-scale  settings. To extend the strong efficiency guarantee to generic strongly convex losses, we propose to iteratively convert the non-quadratic problem into a sequence of quadratic subproblems such that the aforementioned method can be readily applied for optimization. In this way, up to logarithmic factors, \HSDAN~still enjoys an identical sharp bound of IFO for strongly convex problems.

	
	Table~\ref{comparisontable} summarizes the computational complexity (measured by IFO) of \HSDAN~and several representative baselines, including	SGD~\cite{robbins1951stochastic,shamir2011making}, SVRG~\cite{SVRG}, SAGA~\cite{SAGA}, APSDCA~\cite{shalev2014accelerated}, APCG~\cite{lin2014accelerated}, SPDC~\cite{zhang2015stochastic},  Catalyst~\cite{Catalyst},  Varag~\cite{lan2019unified}, AMSVRG~\cite{AMSVRG},   Katyusha~\cite{katyusha},  SCSG~\cite{lei2017less}. In the following, we highlight the advantages of our method over these prior approaches:
	\begin{itemize}
		\item  To achieve  $\epsilon$-optimization-error,~\ie~$\EE[F(\wm)\!-\!F(\wms)]\!\leq\! \epsilon$, the IFO complexity of \HSDAN~on  problem~\eqref{eqn:general}~ is $\mathcal{O}\Big(\!\frac{\kappa^{1.5}\epsilon^{0.75} \log^{\taui{1}}\!(\frac{1}{\epsilon}) + 1}{\epsilon}\!  \wedge\!  \Big(\!\kappa \sqrt{n} \log^{\taui{2}}\!\!\big(\frac{1}{\epsilon}\big) \!+\!n \log^{\taui{3}}\!\big(\frac{1}{\epsilon}\big)\!\Big)\!\Big)$  where $\taui{1}=1.5$, $\taui{2}=1.5$ and $\taui{3}=1$  for quadratic loss and  $\taui{1}=2.25$, $\taui{2}=2.5$ and $\taui{3}=2$ for generic strongly convex loss. In comparison, the IFO complexity bounds of all the compared algorithms except SGD and SCSG scale linearly w.r.t. the data size $n$. As specified in the third column of Table~\ref{comparisontable}, \HSDAN~is superior to these algorithms in large-scale problem settings which are of central interest in big data applications. Compared with SGD, since in most cases, the condition number $\kappa$ is at the order of $\mathcal{O}(1/\mu)$, \HSDAN~improves over SGD by a factor at least $\mathcal{O}\big( \kappa \wedge \frac{1}{\kappa^{0.5}\epsilon^{0.75}}\big)$   (up to logarithm factors). For SCSG, \HSDAN~also shows higher computational efficiency when (1) the optimization error $\epsilon$ is small which corresponds to conditions \ding{172} or \ding{173} in  Table~\ref{comparisontable}; and (2) the data size $n$ is large which corresponds to condition \ding{174}  in  Table~\ref{comparisontable}.
		\item  For the practical setting where $\epsilon=\mathcal{O}(1/\sqrt{n})$ which matches the optimal intrinsic excess error, \HSDAN~has the IFO complexity $\mathcal{O} \left(n^{0.875}\!\log^{1.5}\! \left(n\right) \right)$ for the quadratic loss and $\mathcal{O} \left(n^{0.875}\!\log^{2.25}\! \left(n\right) \right)$ for the generic strongly convex loss. By ignoring  the small logarithm term $\log(n)$,  both complexities of \HSDAN~are lower than the complexity bound $\mathcal{O}\big(n\big)$  of  SGD by a factor $\mathcal{O}\big(n^{0.125}\big)$. Similarly, \HSDAN~respectively improves over APCG and other remaining algorithms, such as  SVRG, Katyusha, Varag  and SCSG,  by  factors of  $\mathcal{O}\big(n^{0.375}\big)$ and $\mathcal{O}\big(n^{0.125}\big)$. These results demonstrate the superior computational efficiency of \HSDAN~ for attaining near-optimal generalization rate of a statistical learning model.
	\end{itemize}

	\setbetter
	\section{Related Work}\label{relatedwork}
	\textbf{Stochastic gradient algorithms.} Gradient descent (GD) \cite{Cauchy1847} method has long been applied to solve ERM and enjoys  linear convergence rate on strongly convex problems. But it needs to compute full gradient per iteration, leading to huge  computation cost on large-scale problems.  To improve  efficiency, incremental gradient algorithms  have  been developed via leveraging the finite-sum structure and have  witnessed tremendous
	progress recently. For instance,  SGD~\cite{robbins1951stochastic,bottou1991stochastic} only evaluates gradient of one (or a minibatch) randomly selected sample at each iteration, which greatly reduces the cost of each iteration and  shows more appealing efficiency than GD on large-scale  problems~\cite{shamir2011making,AMSVRG,hendrikx2019asynchronous,mohammadi2019robustness}. Along this line of research, a variety of variance-reduced variants, such as SVRG~\cite{SVRG},  SAGA~\cite{SAGA}, APSDCA~\cite{shamir2011making}, AMSVRG~\cite{AMSVRG},  SCSG~\cite{lei2017less}, Catalyst~\cite{Catalyst}, Katyusha~\cite{katyusha}, Varag~\cite{lan2019unified}, are developed and have delivered exciting progress such as linear convergence rates on strongly convex problems as opposed to sublinear rates of  vanilla SGD~\cite{shamir2011making}.
	The hybrid stochastic-deterministic gradient descent  method~\cite{friedlander2012hybrid,zhou2018HSGD,zhou2018new,mokhtari2016adaptive,mokhtari2017first} iteratively samples an evolving minibatch of samples for gradient estimation or subproblem construction and works favorably in reducing the computational complexity.
	Our \HSDAN~method differs significantly from these prior algorithms. Based on the Bregman-divergence of the minibatch  function and a hybrid stochastic-deterministic first-order approximation of the original  function, \HSDAN~constructs a variance-reduced minibatch proximal function  which is provably  more efficient. Moreover, \HSDAN~can employ any off-the-shelf algorithms to solve the constructed sub-problems in the inner loop and thus is flexible for implementation. \HSDAN~shares a similar spirit with the DANE method~\citep{shamir2014communication} which also uses a local Bregman-divergence-based function approximation for communication-efficient  distributed quadratic loss optimization. The main difference lies in the way of constructing first-order approximation of the risk function: \HSDAN~employs a novel hybrid stochastic-deterministic approximation strategy which is substantially more efficient than the deterministic strategy as used by DANE.

	\textbf{Generalization and optimization.} In the seminal work of~\citet{bottou2008tradeoffs}, it has been demonstrated that the excess error that measures the generalization performance of an ERM model over a function class can be decomposed into three terms in expectation: an \emph{approximation error} that measures how accurate the function class can approximate the underlying optimum model; an \emph{estimation error} that measures the effects of minimizing ERM instead of population risk; and an \emph{optimization error} that represents the difference between the exact solution and the approximate solution of ERM. Particularly, for the $\ell_2$-regularized convex ERM with linear models as in~\eqref{eqn:general}, its estimation error (or excess risk) has long been studied with a vast body of deep theoretical results established~\cite{shalev2014understanding,hardt2015train,bach2013non,dieuleveut2017harder,zhou2018analysiscnns,zhou2018analysisdnn}. A simple yet powerful tool for analyzing estimation error is the \emph{stability} of an estimator to the changes of training dataset~\cite{bousquet2002stability}. The $\ell_2$-regularized convex ERM has been shown to have uniform stability of order $\mathcal{O}(1/(\mu n))$~\cite{bousquet2002stability}, which then gives rise to the optimal choice $\mu=\mathcal{O}(1/\sqrt{n})$ to balance empirical loss and generalization gap to achieve estimation error  $\mathcal{O}(1/\sqrt{n})$~\cite{shalev2009stochastic,feldman2019high}. This implies that the overall excess error is dominated by $\mathcal{O}(1/\sqrt{n})$. In this sense, it suffices to solve the $\ell_2$-regularized ERM to  optimization error $\mathcal{O}(1/\sqrt{n})$ to match the intrinsic excess error. 

	\vspace{-0.3em}
	\section{Hybrid Stochastic-Deterministic Minibatch Proximal Gradient}
	\label{sect:HSDMPG}
	\vspace{-0.3em}
	In this section, we first introduce the hybrid stochastic-deterministic minibatch proximal gradient (\HSDAN) algorithm for quadratic loss function along with convergence rate and computational complexity analysis. Then, we extend \HSDAN~and its theoretical analysis to generic strongly convex loss functions. 
	
	\subsection{The \HSDAN~method for quadratic loss}\label{resultsofquadratic}
	\subsubsection{Algorithm}
	The \HSDAN~method is outlined in Algorithm~\ref{alg:hsdvrg}. The initial step is to randomly sample a minibatch  $\SSm$ of data points of size $s$ to construct a stochastic approximation
	\begin{equation}\label{approximationfunction}
	\Fms(\wm)=\frac{1}{s}\sum\nolimits_{i\in\SSm} \ell(\wm^\top\xmi{i},\ymi{i})+\frac{\mu}{2}\|\wm\|_2^2
	\end{equation}
	to the original risk function $F(\wm)$ in problem~\eqref{eqn:general}. $\Fms(\wm)$ will be fixed throughout the computational procedure to follow.
	Then in the iteration loop the algorithm iterates between two steps of S1 and S2. In step S1, we uniformly randomly sample a size increasing minibatch $\Smmi{t}$ of samples to  estimate an inexact function $F_{\Smmi{t}}(\wm)= \frac{1}{|\Smmi{t}|} \sum_{i\in\Smmi{t}} \ell(\wm^\top \xmi{i}, \ymi{i}) + \frac{\mu}{2}\|\wm\|_2^2$. Let $\Div_g(\wmi{1},\wmi{2}) = g(\wmi{1}) - g(\wmi{2}) - \langle \nabla g(\wmi{2}), \wmi{1} - \wmi{2}\rangle $ denote the Bregman divergence of a function $g$. Based on $\Fms(\wm)$ and $F_{\Smmi{t}}(\wm)$,   we  construct a variance-reduced minibatch proximal objective $\Pmti{t-1}(\wm)$ to  approximate the objective $F(\wm)$ in~\eqref{eqn:general}, where  $\Pmti{t-1}(\wm)\triangleq $
	\begin{equation*}
	F_{\Smmi{t}}\!(\wmi{t-1}) + \langle \nabla F_{\Smmi{t}}\!(\wmi{t-1}), \wm  -  \wmi{t-1} \rangle  +  \Div_{\Fmts}(\wm,\wmi{t-1}).
	\end{equation*}
	Here $\Div_{\Fmts}(\wm,\wmi{t-1})$ is the Bregman divergence of a regularized loss $\Fmts(\wm)=\Fms(\wm) + \frac{\gamma}{2} \|\wm\|_2^2$ which essentially measures the distance between $\wmi{t}$ and $\wmi{t-1}$ on the current geometry curve estimated on $\Fmts(\wm)$. We define the next iterate as
	{
		\begin{equation}\label{equat:P_t_w}
		\wmi{t}=\arg\min\nolimits_{\wm}\Pmti{t-1}(\wm) = \arg\min\nolimits_{\wm} \Pmi{t-1}(\wm),
		\end{equation}}
	where $\Pmi{t-1}(\wm) \triangleq$
	{
		\begin{equation*}
		\Fms(\wm) + \langle \nabla F_{\Smmi{t}}(\wmi{t-1}) \!-\! \nabla \Fms(\wmi{t-1}), \wm \rangle + \frac{\gamma}{2}\|\wm-\wmi{t-1}\|_2^2.
		\end{equation*}}
	In $\Pmi{t-1}$, its finite-sum structure comes from the initial stochastic approximation $\Fms(\wm)$ and its gradient at $\wmi{t-1}$. Since along with more iterations, the size of $\Smmi{t}$  increases which indicates that the loss $\Pmi{t-1}$ is a variance-reduced loss and will converge to the original loss $F(\wm)$ in problem~\eqref{eqn:general}. Then in step S2, we approximately solve problem~\eqref{equat:P_t_w} via a stochastic gradient optimization method such as SVRG. The principle behind this strategy is that for the initial optimization progress, inexact gradient already can well decrease the loss since the current solution is far from the optimum, while along more iterations, the current solution becomes closer to optimum, requiring more accurate gradient for further reducing the loss function. In this way, our proposed method can well balance the converge speed and the computational cost at each iteration and thus has the potential to achieve improved overall computational efficiency.
	\citet{shamir2014communication} has proposed the DANE method which uses a similar local Bregman divergence based regularization for distributed quadratic optimization problems. Our method improves upon DANE in two aspects: 1) we use variance-reduction techniques to reduce the overall computational complexity, and 2) \HSDAN~is applicable not only to quadratic problems but also to generic strongly convex problems with about the same computational complexity as discussed in Sec.~\ref{proofofgeneralloss}.

	\begin{algorithm}[t]
		\caption{Hybrid Stochastic-Deterministic Minibatch Proximal Gradient (\HSDAN) for quadratic loss.}
		\label{alg:hsdvrg}
		\begin{algorithmic}
			\STATE	\textbf{Input:} {initialization $\wmi{0}$, regularization constant $\gamma$ in~\eqref{equat:P_t_w}, optimization error~$\vapi{t}$.}
			\STATE  \textbf{Initialization:}	Uniformly  randomly sample a data batch $\SSm$ of size $s$  to form $\Fms(\wm)$ in~\eqref{approximationfunction}.
			
			\FOR{$t=1, 2, \ldots, T$}{
				
				\STATE	(S1) Uniformly randomly   sample a minibatch  $\Smmi{t}$ to form  $F_{\Smmi{t}}\!(\wm)\!=\!\frac{1}{|\Smmi{t}|}\!\sum_{i\in\Smmi{t}} \!\ell(\wm^\top\xmi{i},\ymi{i})\!+\!\frac{\mu}{2}\|\wm\|_2^2$ and compute $\nabla F_{\Smmi{t}}\!(\wmi{t-1})$ to construct loss $\Pmi{t-1}(\wm)$ in~\eqref{equat:P_t_w}.
				\STATE	(S2)	Optimize the subproblem~\eqref{equat:P_t_w}, \eg~via SVRG, to obtain $\wmi{t}$ that satisfies $\|\nabla \Pmi{t-1}(\wmi{t})\|_2\le \vapi{t}$.
			}
			\ENDFOR
			\STATE   \textbf{Output:}	{$\wmi{T}$.}
		\end{algorithmic}
	\end{algorithm}
	
	\subsubsection{Convergence and  complexity analysis}\label{quatraticlosscomplexity}
	
	We first introduce two necessary definitions, namely~strong convexity and Lipschitz smoothness,  which are conventionally used in the analysis of convex optimization methods~\cite{shamir2011making,SVRG}.
	\begin{definition}[Strong Convexity and Smoothness]\label{def:strong_smooth}
		A differentiable function $g(\wm)$ is said to be $\mu$-strongly-convex and $L$-smooth if $\forall \wmi{1}, \wmi{2}$, it satisfies
		\begin{equation*}
		\frac{\mu}{2}\|\wmi{1} - \wmi{2}\|_2^2 \le\Div_g(\wmi{1},\wmi{2})\le \frac{L}{2}\|\wmi{1} - \wmi{2}\|_2^2.
		\end{equation*}
		where  $\Div_g(\wmi{1},\wmi{2}) = g(\wmi{1}) - g(\wmi{2}) - \langle \nabla g(\wmi{2}), \wmi{1} - \wmi{2}\rangle $.
	\end{definition}

	For brevity, let  $\Hm$ be the Hessian matrix of the quadratic function $F(\wm)$ and  $ \elli{i}(\wm) = \ell(\wm^{\top}\xmi{i},\ymi{i}) + \frac{\mu}{2}\|\wm\|_2^2$. Denote $\|\wm\|_{\Hm}\!=\!\sqrt{\wm^\top\! \Hm \wm}$. In the analysis to follow, we always suppose that $\|\xmi{i}\|\leq \rx$,$\forall i$, which  generally holds for natural data analysis, \eg, in computer vision and signal processing. We summarize our main result in Theorem~\ref{thrm:quadratic_hsdmpg}  which shows the linear convergence rate of \HSDAN~for quadratic problems. See proof in Appendix~\ref{proofoftheorem1}.
	

	\begin{thm}\label{thrm:quadratic_hsdmpg}
		Assume each loss $\ell(\wm^\top\! \xmi{i},\ymi{i})$ is quadratic and $\Ls$-smooth w.r.t. $\wm^\top\xmi{i}$, and $\sup_{\wm}\!\frac{1}{n}\!\sum_{i=1}^n $ $\|\Hm^{-1/2} (\nabla F(\wm) - \nabla \elli{i}(\wm)) \|_2^2 \le \nu^2$.  By setting $ \gamma= (\sqrt{ \log(d)} +\sqrt{2} )\Ls r^2/\sqrt{s}$, $\varepsilon_t = \frac{\mu^{1.5}}{4(\mu + 2\gamma)}\exp\big(-\frac{\mu(t-1)}{2(\mu + 2\gamma)}\big)$,  $|\Smmi{t}| = \frac{16\nu^2(\mu+2\gamma)^2}{\mu^2} \! \exp\!\big(\frac{\mu t}{2(\mu+2\gamma)}\big) \wedge n$,  where $d$ is the problem dimension,
		the sequence $\{\wmi{t}\}$ produced by Algorithm~\ref{alg:hsdvrg} satisfies
		\begin{equation*}
		\mathbb{E}[F(\wmi{t}) - F(\wms)] =\mathsmaller{\frac{1}{2}}\mathbb{E}[\|\wmi{t} -\wms\|_{\Hm}^2]
		\leq \zeta  \exp\big(-\mathsmaller{\frac{\mu t}{\mu+2\gamma}}\big),
		\end{equation*}
		where $\zeta\!=\!  \frac{1}{2}\left( \|\wmi{0} \!-\! \wms\|_{\Hm} \!+\!  \frac{1}{2}\right)^2 \!+\! \frac{5}{8} $.
	\end{thm}
	
	The main message conveyed by Theorem~\ref{thrm:quadratic_hsdmpg} is that \HSDAN~enjoys linear convergence rate on the quadratic loss when we use evolving size of the minibatch $\Smmi{t}$.  Note here we only assume   each loss $\ell(\wm^\top\! \xmi{i},\ymi{i})$ is $\Ls$-smooth w.r.t. $\wm^\top\xmi{i}$. This assumption is  much milder than the smoothness assumption on the function $F(\wm)$ w.r.t. $\wm$ which is used in other algorithm analysis, such as SGD and  SVRG. The assumption that $\sup_{\wm}\!\frac{1}{n}\!\sum_{i=1}^n\|\Hm^{-1/2} (\nabla F(\wm) - \nabla \elli{i}(\wm)) \|_2^2$ $ \le \nu^2$ in \HSDAN~is mild, which requires the variance of stochastic gradient under the Hessian matrix is bounded. Such an assumption is analogous to the one used in analysis of SGD that  imposing the bounded-variance assumption on  stochastic gradient, namely,  $ \frac{1}{n}\!\sum_{i=1}^n \|\nabla F(\wm) - \nabla \elli{i}(\wm) \|_2^2 $.

	Based on this result, we further analyze the  computational complexity of \HSDAN~to   better understand its overall efficiency in computation. At each iteration,  we use the SVRG method solve the inner-loop subproblem~\eqref{equat:P_t_w} because it only accesses the first-order information of the objective function and is efficient. Following~\citep{SVRG,zhang2015stochastic,zhou2019Riemannian,shen2019stochastic}, we employ the incremental first order oracle (IFO) complexity as the computation complexity metric for solving the finite-sum  solving problem~\eqref{eqn:general}.
	\begin{definition}\label{def:IFO}
		An IFO takes an index $i \in [n]$ and a point $(\xmi{i},\ymi{i})$, and returns the pair $( \elli{i}(\wm) ,\nabla  \elli{i}(\wm) )$.
	\end{definition}
	The IFO complexity can  accurately reflect the overall computational performance of a first-order
	algorithm, as objective value and gradient evaluation usually dominate the  per-iteration complexity. Based on these preliminaries, we summarize our main result on the computation complexity of \HSDAN~in Corollary~\ref{thrm:quadratic_hsdmpg_ifo} with proof provided in Appendix~\ref{append:proof_of_dane_hb_ifo}.
	
	\begin{cor}[Computation complexity of HSDMPG for quadratic loss]\label{thrm:quadratic_hsdmpg_ifo}
		Suppose that the assumptions in  Theorem~\ref{thrm:quadratic_hsdmpg} hold and the inner-loop subproblems are solved via SVRG, then  the IFO complexity of \HSDAN~on the quadratic loss  to achieve $\mathbb{E}[F(\wmi{t}) - F(\wms) ] \le  \epsilon$ is of the order  $	 \mathcal{O} \Big( \big(1+ \frac{\kappa^3 \log^{1.5}(d)}{s^{1.5}}\big)\frac{\nu^2}{\epsilon} \bigwedge\big(1+ \frac{\kappa  \log^{0.5}(d)}{s^{0.5}}\big) n \log \big(\frac{1}{\epsilon} \big)  +  \kappa \sqrt{s\log(d)}   \log^2\big(\frac{1}{\epsilon}\big) \Big),$ 
		where $\kappa=L/\mu$  denotes the conditional number.
	\end{cor}
	
	According to Corollary~\ref{thrm:quadratic_hsdmpg_ifo}, by choosing $s$ as $s= \frac{\kappa \nu \log^{0.5}(d)}{\epsilon^{0.5}\log(1/\epsilon)}\wedge n$ or $s= \frac{n}{\log(1/\epsilon)}$ and ignoring the constant $\nu$ and the logarithm factor $\log(d)$ of the problem dimension $d$,  the IFO complexity of \HSDAN~is at the order of
		\begin{equation*}\mathcal{O}\Big(\!\frac{\kappa^{1.5}\epsilon^{0.75} \log^{1.5}\!(\frac{1}{\epsilon}) + 1}{\epsilon}\!  \wedge\!  \Big(\!\kappa \sqrt{n} \!\log^{1.5}\!\!\big(\frac{1}{\epsilon}\big) \!+\!n \!\log\!\big(\frac{1}{\epsilon}\big)\!\Big)\!\Big).
			\end{equation*}
	One may compare such a complexity with the state-of-the-arts listed in Table~\ref{comparisontable}. Compared with those algorithms in the table whose IFO complexity scales linearly with  the data size $n$, \eg~SVRG, APCG, Katyusha and AMSVRG, the proposed \HSDAN~has data-size-independent IFO complexity and can outperform them for large-scale learning problems where the data size $n$ could be huge. To be more precise, the third column of Table~\ref{comparisontable} summarizes the conditions under which \HSDAN~outperforms these algorithms in terms of  computational complexity. For the algorithms whose  IFO complexity  does not depend on $n$, namely SGD and SCSG,  \HSDAN~also enjoys substantially lower complexity in most cases.  Concretely,
	since $\kappa$ is typically at the order of $\mathcal{O}\big(1/\mu\big)$, when $\kappa\leq\epsilon^{1.5}$ which holds for moderately larger $\kappa$,   \HSDAN~improves over SGD by a factor at least $\mathcal{O}\big( \kappa \wedge \frac{1}{\kappa^{0.5}\epsilon^{0.75}}\big)$  (up to the logarithmic factor).  As for  SCSG, \HSDAN~also achieves  higher  efficiency when (1) the optimization error is small which corresponds to conditions \ding{172}   in the third column of Table~\ref{comparisontable}, (2) the sampler size $n$ is large which corresponds to condition \ding{173}.  These results show that \HSDAN~is well suited for solving large-scale learning problems.
	
	From the perspective of generalization, we are particularly interested in the computational complexity of \HSDAN~for optimizing the $\ell_2$-ERM model~\eqref{eqn:general} to its  intrinsic excess error bound which characterizes the generalization performance of the model. As reviewed in Section~\ref{relatedwork}, the excess  error of the considered $\ell_2$-ERM model is typically of order $\mathcal{O}(1/\sqrt{n})$. Accordingly, one only needs to solve the optimization problem to the optimization error $\epsilon=\mathcal{O}(1/\sqrt{n})$~\cite{bottou2008tradeoffs,shalev2009stochastic}.  Moreover, to accord with this intrinsic excess error bound, the regularization constant $\mu$ should also be at the order of $\mathcal{O}(\frac{1}{\sqrt{n}})$. In this way, the condition number $\kappa$ could scale as large as $\mathcal{O}(\sqrt{n})$. Based on these results and Corollary~\ref{thrm:quadratic_hsdmpg_ifo}, we can derive the IFO complexity bound of \HSDAN~for this case in Corollary~\ref{optimizationerrorcomplexity}.

	\begin{cor}\label{optimizationerrorcomplexity}
		Suppose  that the assumptions  in Corollary~\ref{thrm:quadratic_hsdmpg_ifo} hold. By setting $s\!=\!\mathcal{O}\big(\frac{\nu n^{0.75}\! \log^{0.5}\!(d)}{\log(n)}\big)$,  the IFO complexity of \HSDAN~on the quadratic loss to achieve $\mathbb{E}[F(\wmi{t}) \!-\!F(\wms) ] \!\le\! \frac{1}{\sqrt{n}}$ is at the order  of  $\mathcal{O}\! \left(\nu^{0.5}n^{0.875} \!\log^{0.75}\!(d)\!\log^{1.5} \!\left(n\right) +\nu^2 n^{0.5}\right)\!.$
	\end{cor}

	See its proof in Appendix~\ref{proofofoptimizationerrorcomplexity}.  From Corollary~\ref{optimizationerrorcomplexity}, one can observe that the IFO complexity of \HSDAN~for quadratic problems is  at the order  of  $\mathcal{O} \left(n^{0.875} \log^{1.5}\left(n\right) \right)$. It means that \HSDAN~can reach the intrinsic excess error $\mathcal{O} \big(1/\sqrt{n}\big)$ with strictly less than a single pass over the entire training dataset. In comparison, we can observe from Table~\ref{comparisontable} that in the same practical setting, SGD and APCG have IFO complexity $\mathcal{O} \left(n\right)$ and $\mathcal{O} \left(n^{1.25}\log(n)\right)$ respectively. By ignoring the logarithm factor $\log(n)$ which is much smaller than $n$ for large-scale learning problems, \HSDAN~improves over these two methods by factors   $\mathcal{O}\big(n^{0.125}\big)$ and $\mathcal{O}\big(n^{0.375}\big)$, respectively. The IFO complexity of all  other algorithms in Table~\ref{comparisontable}, including SVRG, SCSG, SPDC, APSDCA, AMSVRG, Catalyst, Katyusha and Varag,  are all at the order of  $\mathcal{O} \left(n\log\left(n\right)\right)$. Similarly, by ignoring the logarithmic factors, \HSDAN~has lower IFO complexity than these algorithms by a factor $\mathcal{O}\big(n^{0.125}\big)$. To summarize this group of results comparison,  \HSDAN~would be significantly superior to all these state-of-the-arts when solving quadratic optimization problems to intrinsic excess error.

	\begin{algorithm}[tb]\caption{Hybrid Stochastic-Deterministic Minibatch Proximal Gradient (\HSDAN) on the generic loss.}
		\label{alg:generalalgorithm}
		\begin{algorithmic}
			\STATE {\bfseries Input:}  {Regularization constant $\gamma$ and initialization $\wmi{0}$.}
			\FOR{$t=1, 2, \ldots,T$}{
				
				\STATE (S1) Construct a finite-sum quadratic function $\Qmi{t-1}(\wm)$ in Eqn.~\eqref{eqn:quadratic_approx} to approximate  $F(\wm)$ at $\wmi{t-1}$.
				
				\STATE	(S2) Run Algorithm~\ref{alg:hsdvrg} with regularization constant $\gamma$ and initialization $\wmi{t-1}$ to minimize the finite-sum function $\Qmi{t\!-\!1}(\wm)$  such that $\Qmi{t\!-\!1}(\wmi{t})\!\le\! \min_{\wm}\! \Qmi{t\!-\!1}(\wm)\!+\! \vapis{t}.$
			}
			\ENDFOR
			\STATE\textbf{Output:} {$\wmi{T}$.}
		\end{algorithmic}
	\end{algorithm}
	
	\setbetter
	\subsection{Algorithm for generic convex loss function}\label{proofofgeneralloss}
	
	The computational complexity guarantees established in the previous section  are only applicable to quadratic loss function. In order to extend these results to non-quadratic convex loss function, we  apply a quadratic approximation strategy  
	to convert the original non-quadratic problem into a sequence of quadratic optimization sub-problems such that each of the subproblem can be optimized by \HSDAN. More specifically, suppose that the loss function $\ell(\wm^{\top}\xm,\ym)$ is twice differentiable w.r.t. $\wm^\top\xm$ and is $\rhos$-smooth w.r.t. $\wm^\top\xm$. Then we can verify that $\nabla^2 F(\wm) =\frac{1}{n}\sum_{i=1}^n \ell''(\wm^\top \xmi{i}, \ymi{i}) \xmi{i} \xmi{i}^\top + \mu \Imm \preceq \Hmss \triangleq \frac{\rhos}{n}\sum_{i=1}^n \xmi{i} \xmi{i}^\top + \mu \Imm$ for all $\wm$. Therefore, at each iteration, we construct an upper bound of the second-order Taylor expansion of $F$ at $\wmi{t-1}$ as expressed by $  \Qmi{t-1}(\wm)\triangleq$
	{
		\begin{equation}\label{eqn:quadratic_approx}
		\!F(\wmi{t-1}) \!+\! \langle \nabla F(\wmi{t-1}), \wm \!-\! \wmi{t-1}\rangle \!+\! \Deltai{t-1}(\wm),
		\end{equation}}
	where $\Deltai{t-1}(\wm) = \frac{1}{2} (\wm - \wmi{t-1})^\top \Hmss (\wm - \wmi{t-1})$. The finite-sum structure in $\Qmi{t-1}(\wm)$ comes from $ \nabla F(\wmi{t-1})=\frac{1}{n}\sum_{i=1}^n \nabla \ell(\wm^\top\xmi{i},\ymi{i}) + \mu \wm$ and $\Hmss$. Thus we can estimate $\wmi{t}$ by applying \HSDAN~to the quadratic function $\Qmi{t-1}(\wm)$ with a warm-start initialization $\wmi{t-1}$ such that
	{
		\begin{equation}\label{eqn:epsilon}
		\Qmi{t-1}(\wmi{t})\le \min\nolimits_{\wm} \Qmi{t-1}(\wm)+\vapis{t}.
		\end{equation}}
	The above nested-loop computation procedure is summarized in Algorithm~\ref{alg:generalalgorithm}. We remark that when computing the gradient of $\Qmi{t-1}(\wm)$, we can compute the gradient associated with $\Hmss$ at the point $\wm$ as $\Hmss(\wm-\wmi{t-1}) =\frac{\rhos}{n}\sum_{i=1}^n (\xmi{i}^\top (\wm-\wmi{t-1}) )  \xmi{i} + \mu (\wm-\wmi{t-1}) $ which only computes the inner-product $\xmi{i}^\top (\wm-\wmi{t-1})$ without explicitly computing $\Hmss$. In this way, the computational cost of each stochastic gradient associated with $\Hmss$ is actually much cheaper than that of computing stochastic gradient of  $\nabla F(\wmi{t-1})$, since the former only involves vector products and the later one is usually complicated, \eg~involving the exponential computation in logistic regression. Then  we establish Theorem~\ref{lemma:outer_loop_convergence} to guarantee  the  convergence of Algorithm~\ref{alg:generalalgorithm}  and analyze its computational complexity. See Appendix~\ref{append:proof_of_generalloss_ifo} for a proof of this main result.
	\begin{thm}[Convergence rate and computation complexity of HSDMPG for generic  loss]\label{lemma:outer_loop_convergence}
		Suppose that each  loss function $\ell(\wm^\top\xm,\ym)$ is $\rhos$-smooth and $\sigma$-strongly convex w.r.t. $\wm^\top\xm$. By setting $\vapis{t} = \frac{\sigma}{2\rhos }\exp\left(-\frac{\sigma}{2\rhos}t\right)$, the   sequence $\{\wmi{t}\}$ produced by Algorithm~\ref{alg:generalalgorithm} satisfies
			\begin{equation*}
			F(\wmi{t})   - F(\wms) \leq \exp\Big(- \frac{\sigma t}{2\rhos} \Big)\big(1+ F(\wmi{0}) - F(\wms)\big).
			\end{equation*}
			Suppose the assumptions in Corollary~\ref{thrm:quadratic_hsdmpg_ifo} hold. Then by setting $\kappa=\frac{\rhos}{\mu}$ the IFO complexity of Algorithm~\ref{alg:generalalgorithm} to achieve $\EE\left[F(\wmi{t}) - F(\wms)\right] \le \epsilon$ is at the order of 			$	\mathcal{O}\Big(  \Big(1+ \frac{\kappa^3 \log^{1.5}(d)}{s^{1.5}}\Big) \frac{\rhos\nu^2}{\sigma \epsilon}  \bigwedge\Big(1+ \frac{\kappa  \log^{0.5}(d)}{s^{0.5}}\Big) \frac{L^3 n}{\sigma^3}  \log^2\left(\frac{1}{\epsilon}\right) + \frac{\rhos^2\sqrt{s\log(d)}}{\sigma\mu}\log^3\left(\frac{1}{\epsilon}\right)   \Big).$
%
	\end{thm}
	
	\begin{figure*}[tb]
		\begin{center}
			\setlength{\tabcolsep}{0.0pt}
			\begin{tabular}{cccccc}
				{\hspace{-2pt}}
				\includegraphics[width=0.25\linewidth]{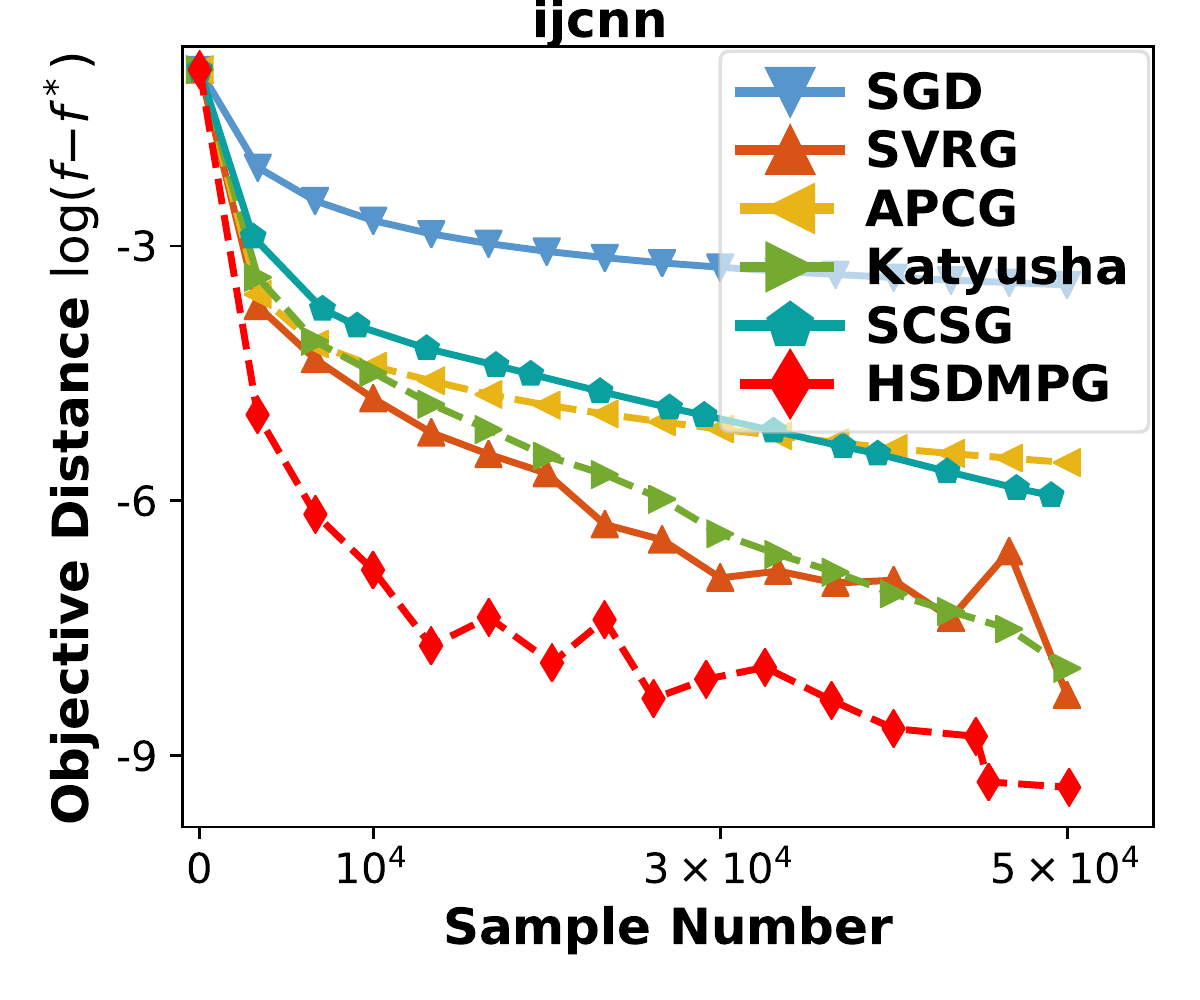}&
				\includegraphics[width=0.25\linewidth]{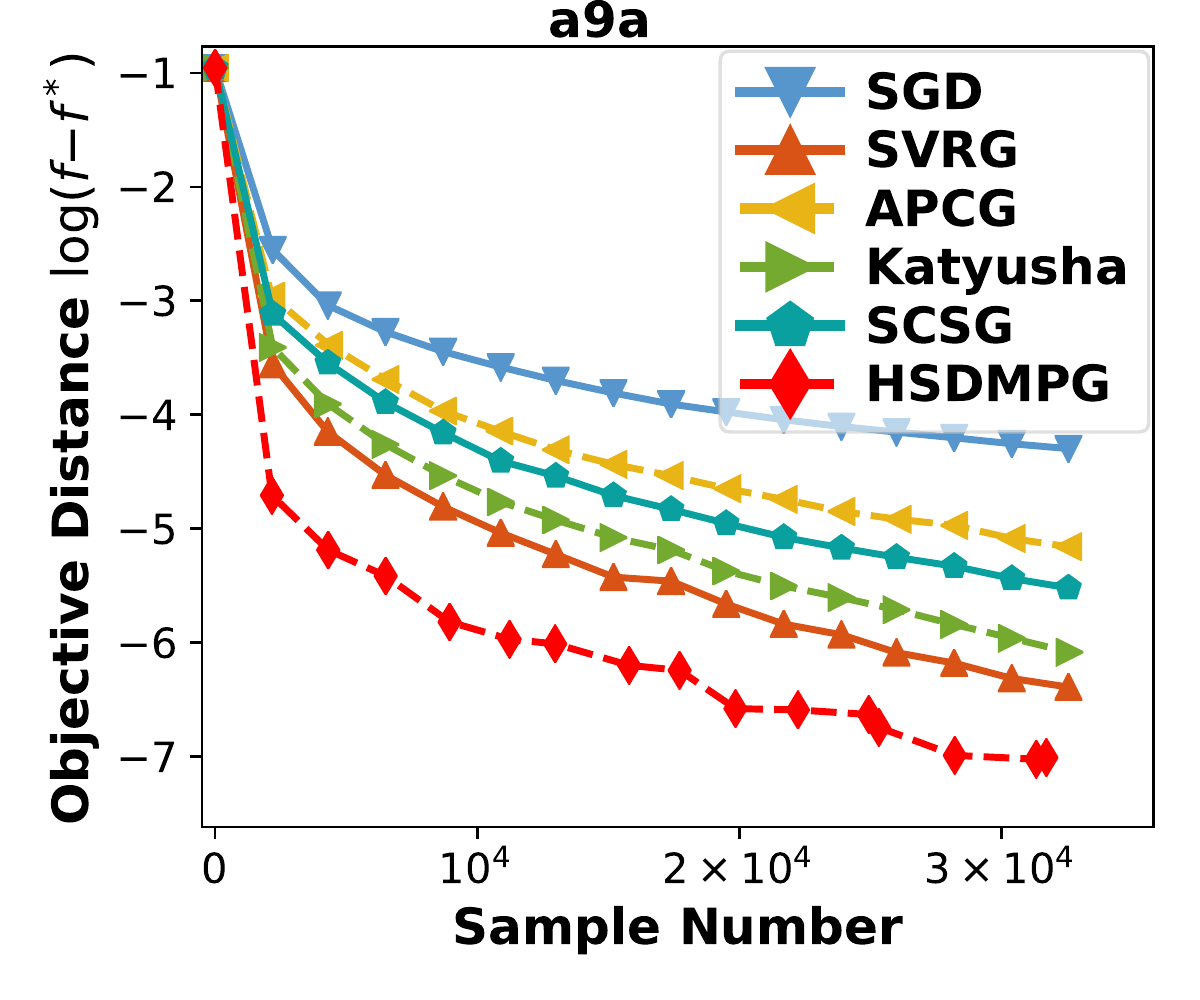}&
				\includegraphics[width=0.25\linewidth]{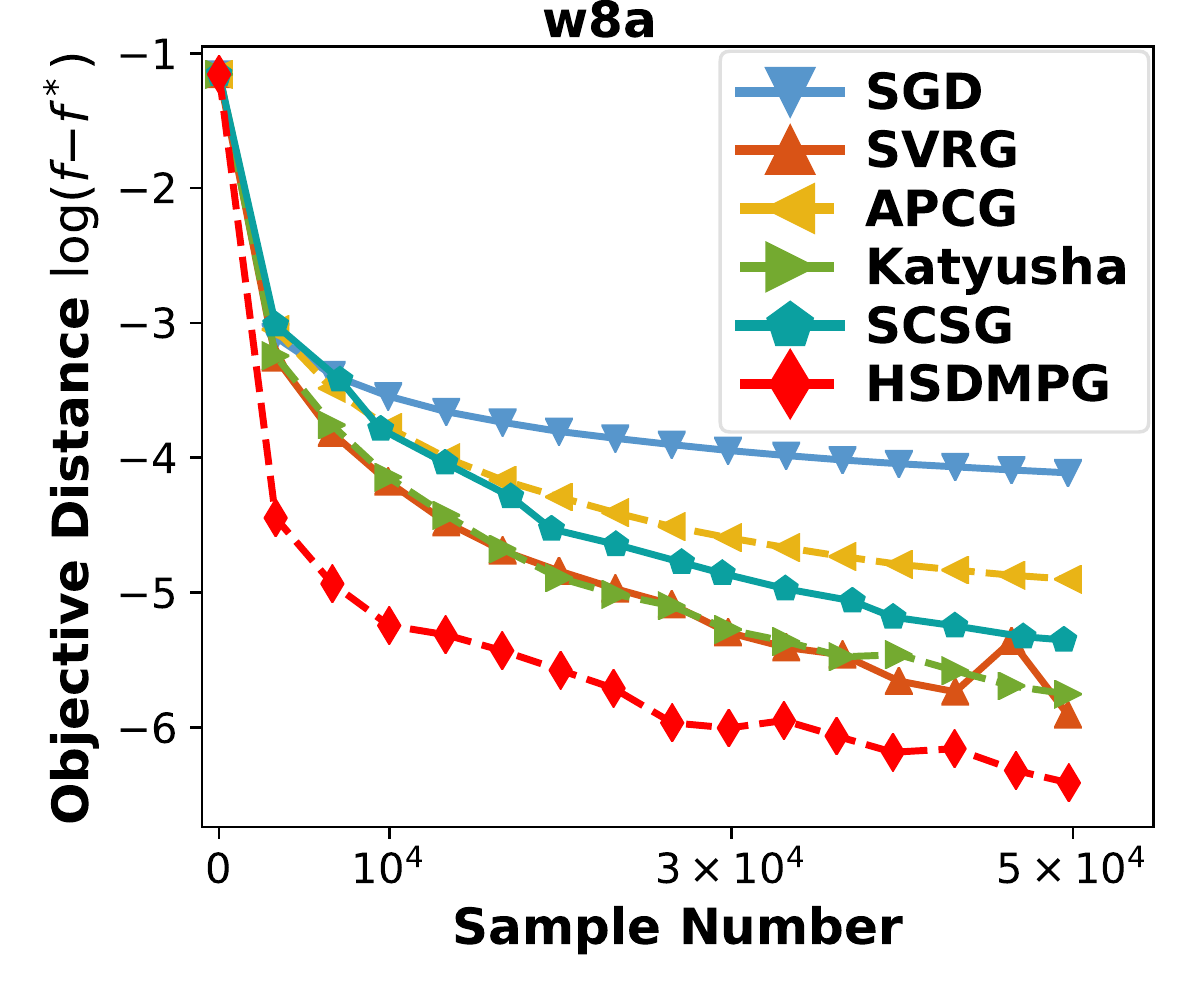}&
				\includegraphics[width=0.25\linewidth]{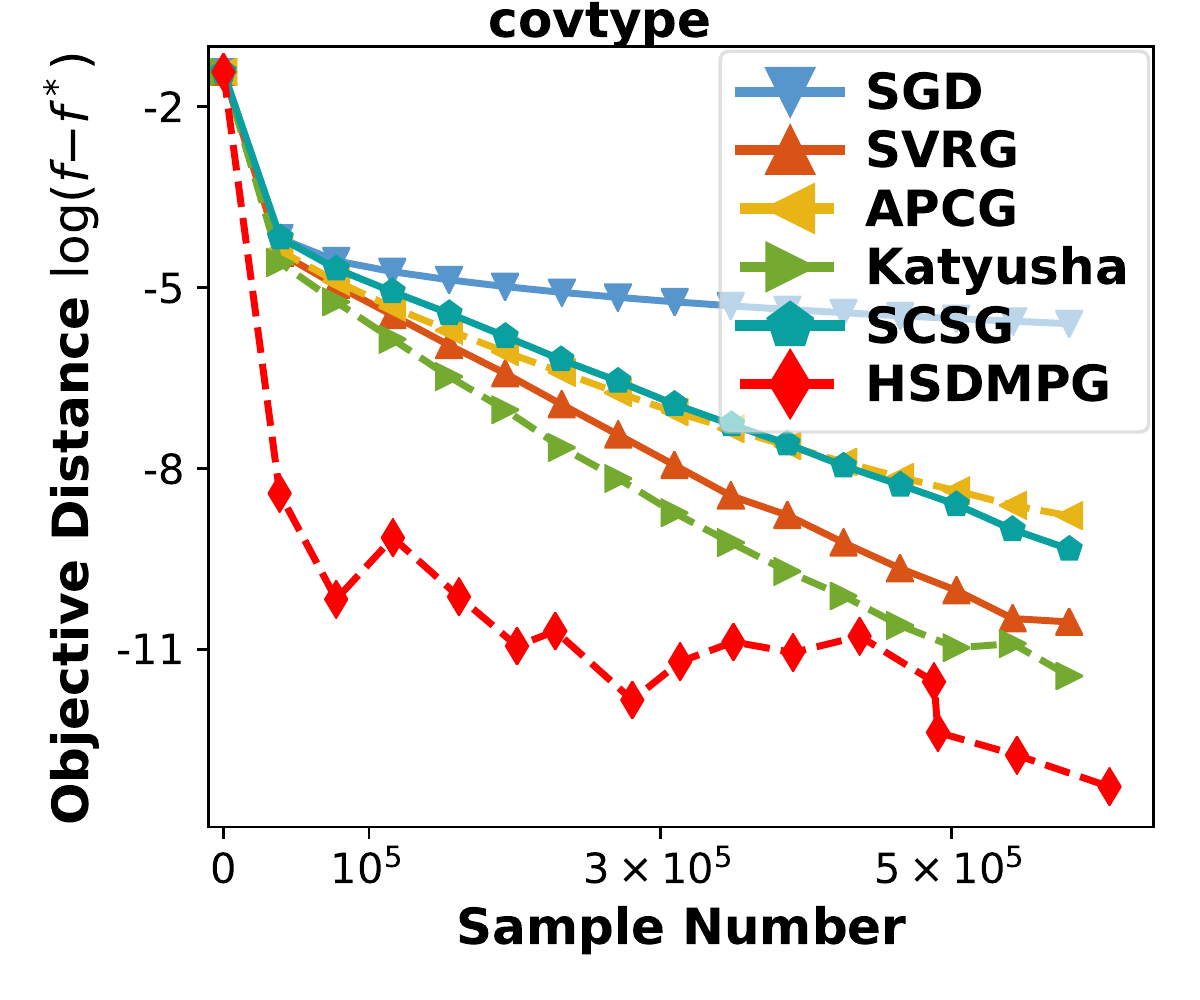}\\
			\end{tabular}
		\end{center}
		\vspace{-1.4em}
		\caption{Single-epoch processing: stochastic gradient algorithms process data a single pass on quadratic problems.
		}
		\label{illustration_convergence}
		\vspace{0.4em}
	\end{figure*}

	\begin{figure*}[tb]
		\begin{center}
			\setlength{\tabcolsep}{0.0pt}
			\begin{tabular}{cccc}
				{\hspace{-2pt}}
				\includegraphics[width=0.25\linewidth]{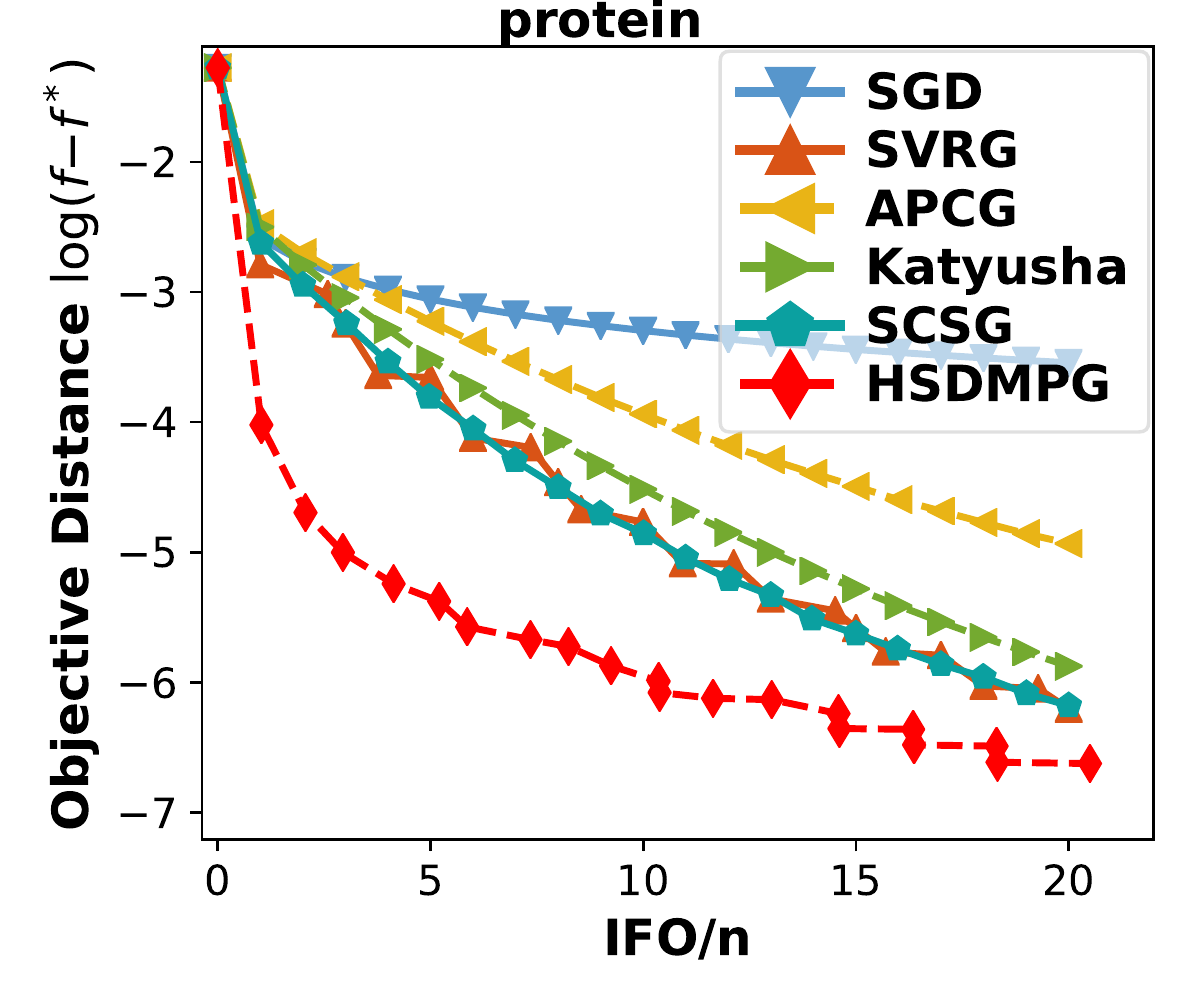}&
				\includegraphics[width=0.25\linewidth]{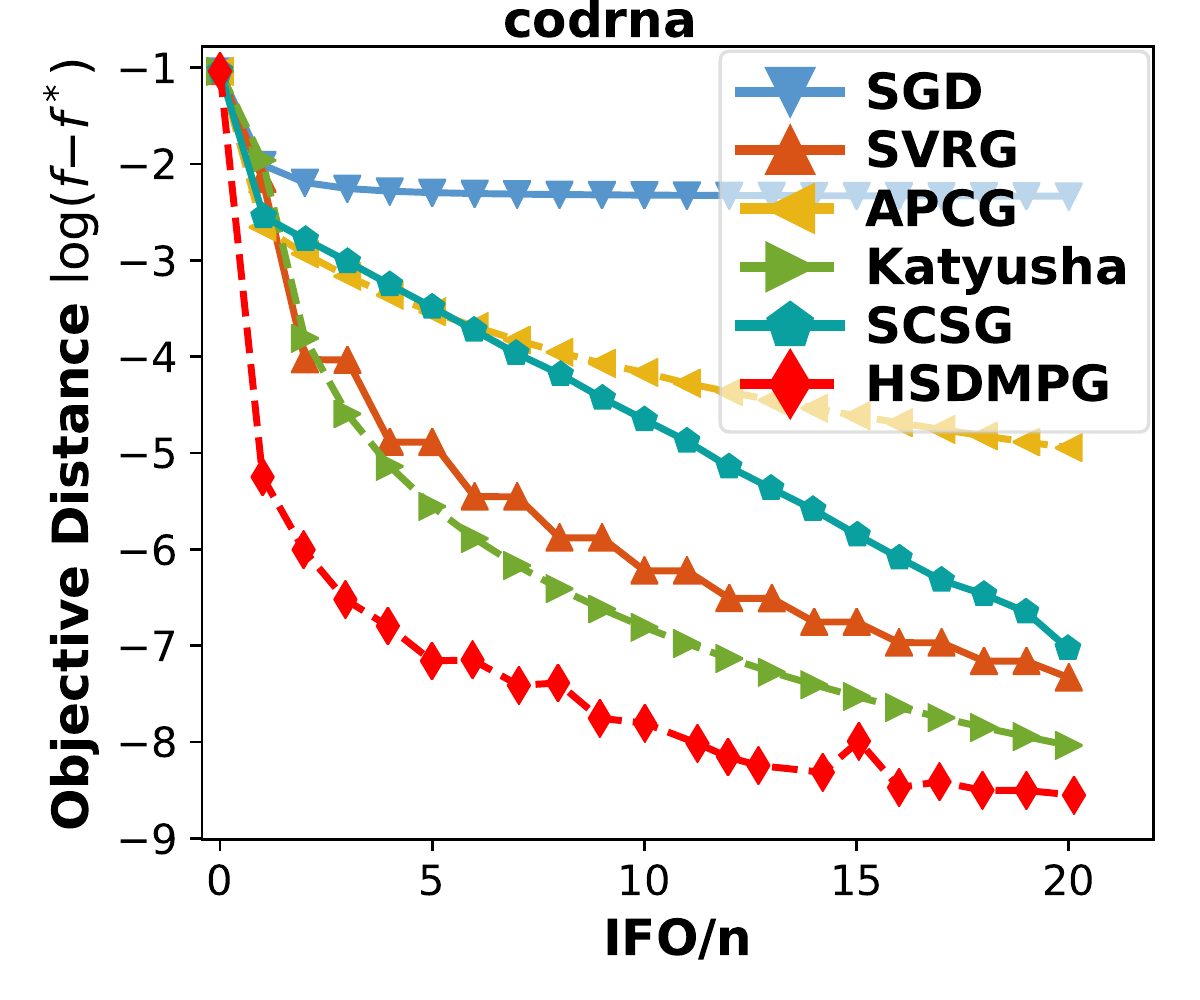}&
				\includegraphics[width=0.25\linewidth]{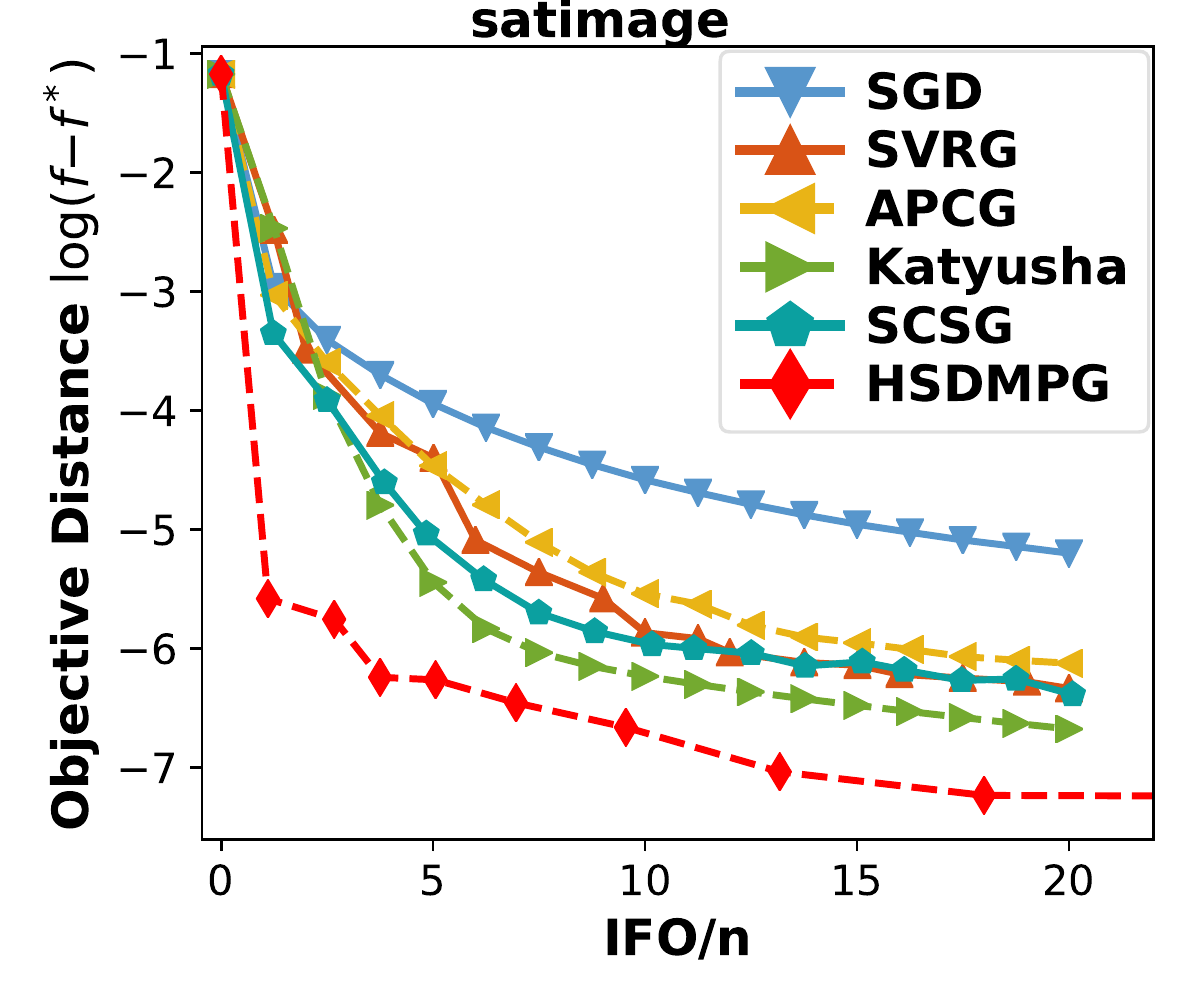}&
				\includegraphics[width=0.25\linewidth]{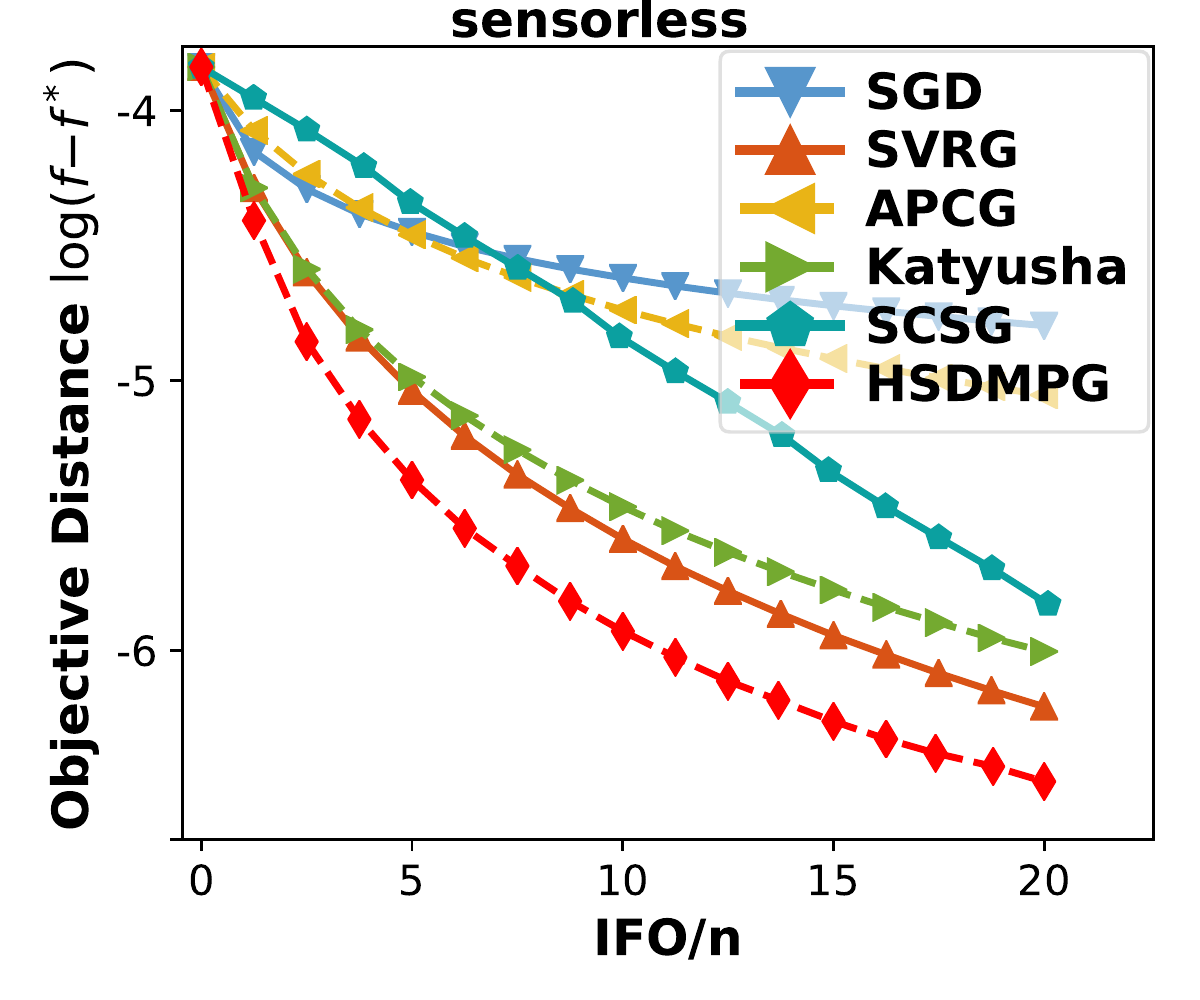}\\
			\end{tabular}
			\begin{tabular}{cc}
				{\hspace{-2pt}}
				\includegraphics[width=0.5\linewidth]{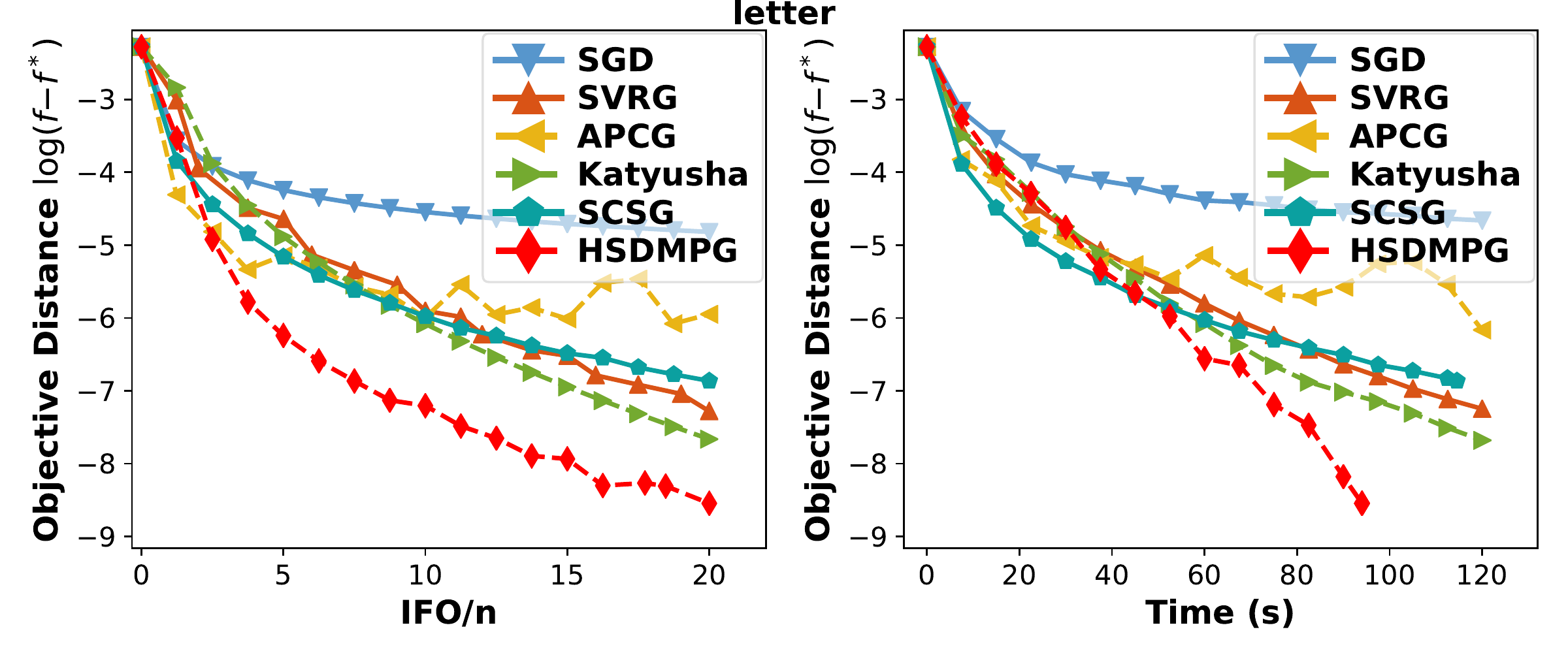}&
				\includegraphics[width=0.5\linewidth]{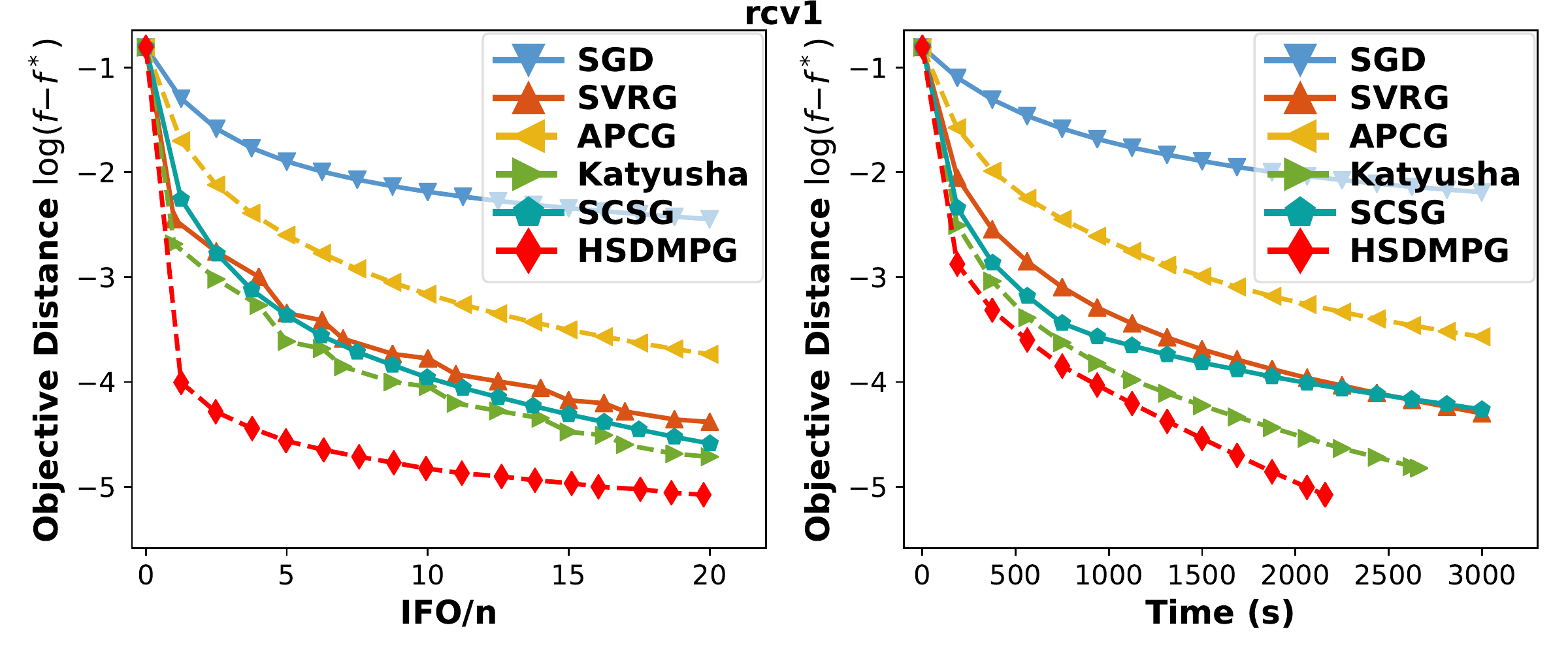}\\
			\end{tabular}
		\end{center}
		\vspace{-1.4em}
		\caption{Multi-epoch processing: stochastic gradient algorithms process data multiple pass on quadratic problems.}
		\label{illustration_convergence2}
	\end{figure*}

	Theorem~\ref{lemma:outer_loop_convergence} suggests that the objective $F(\wmi{t})$ converges linearly to the optimum $F(\wms)$ with rate $\exp(-\frac{\sigma}{2L})$. Note that $\sigma$ is the strong convexity parameter of the loss function $\ell(\wm^\top\xm,\ym)$ w.r.t. $\wm^\top\xm$ instead of $\wm$ which is usually not relying on data scale for widely used loss functions such as the logistic loss~\cite{yuan2019convergence} and thus leads to fast outer-loop convergence rate. In contrast, the strong convexity parameter $\mu$ of the risk function $F$ is typically set at the order of $\mathcal{O}\big(1/\sqrt{n}\big)$ so as to match the intrinsic excess error.
	
	In terms of computational complexity,  by choosing the proper value of $s$ from $s= \frac{ \nu \kappa \log^{0.5}(d)}{ \epsilon^{0.5}\log^{1.5}(1/\epsilon)}\wedge n$ and $s=\frac{\mu L \kappa n}{\sigma^2 \log(1/\epsilon)} \wedge n$  in Algorithm~\ref{alg:hsdvrg},  the IFO complexity of \HSDAN~for generic convex loss can be shown to scale as
		\begin{equation*}
		\mathcal{O}\Big(\!\frac{\kappa^{1.5}\epsilon^{0.75}\! \log^{2.25}\!(\frac{1}{\epsilon}) \!+ \! 1}{\epsilon}   \wedge   \Big(\!\kappa \sqrt{n}  \log^{2.5}\!\big(\frac{1}{\epsilon}\big) \!+n \log^2\!\big(\frac{1}{\epsilon}\big)\!\Big)\!\Big)	
		\end{equation*}
	Compared with the methods listed in Table~\ref{comparisontable},  one can observe  that for generic strongly convex problems, \HSDAN~enjoys lower computational complexity  than all the compared algorithms except SGD and SCSG for large-scale learning problems where the sample number $n$ is sufficiently large to satisfy the conditions in the third column  of Table~\ref{comparisontable}. Similar to the results on quadratic loss,  \HSDAN~improves over SGD by a factor at least $\mathcal{O}\big( \kappa \wedge \frac{1}{\kappa^{0.5}\epsilon^{0.75}}\big)$. So when the optimization error $\epsilon$  is very small or the condition number $\kappa$ is   large, \HSDAN~will be much more efficient than SGD. For SCSG,  \HSDAN~is of higher efficiency in two regimes, namely (1) the optimization error is small which corresponds to conditions \ding{172} or \ding{173} in Table~\ref{comparisontable}, and (2) the sampler number $n$ is large which corresponds to condition \ding{174}. These results show the advantages  \HSDAN~in solving large-scale strongly-convex learning problems.

	\begin{figure*}[tb]
		\begin{center}
			\setlength{\tabcolsep}{0.0pt}
			\begin{tabular}{cccc}
				{\hspace{-2pt}}
				\includegraphics[width=0.25\linewidth]{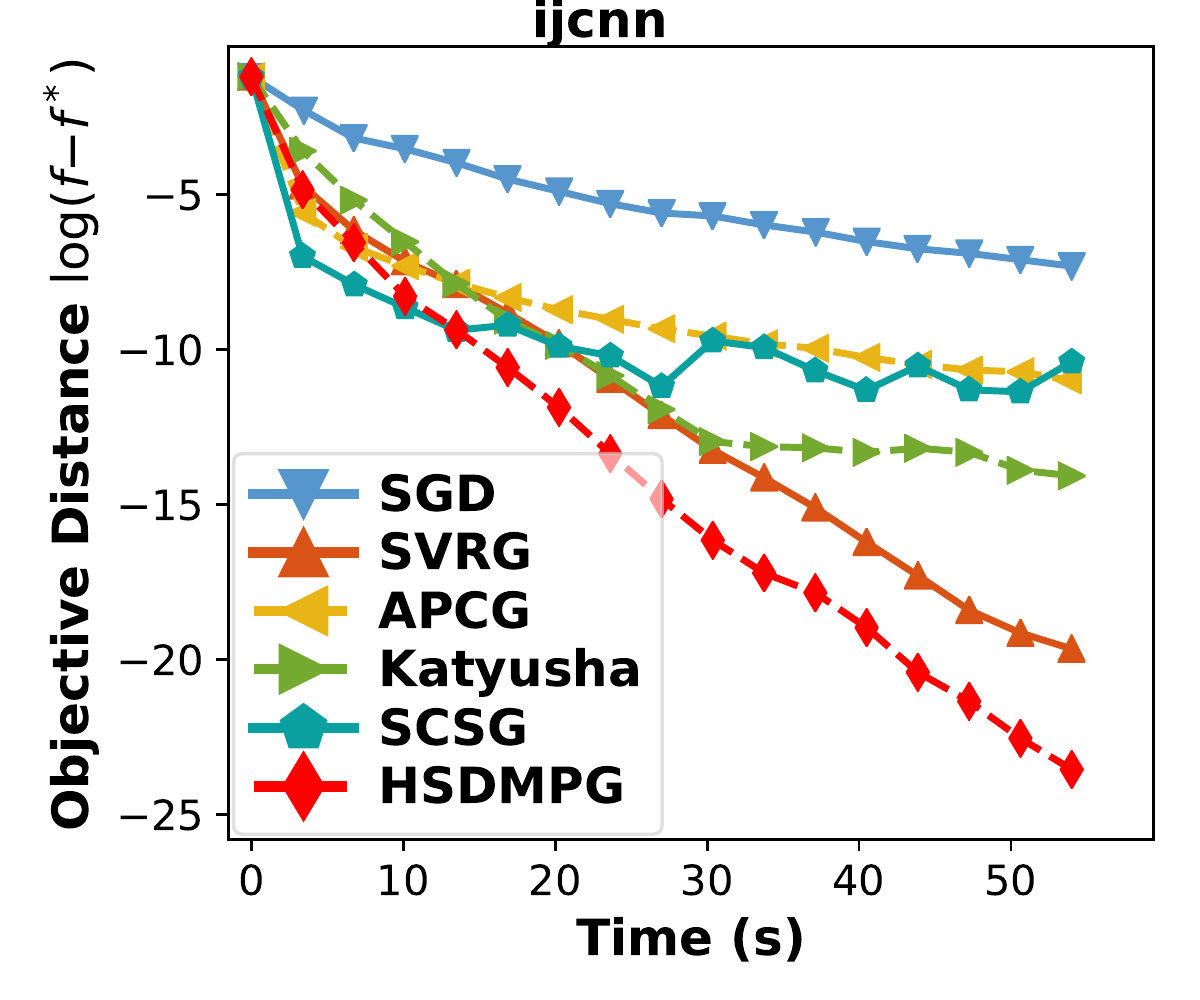}&
				\includegraphics[width=0.25\linewidth]{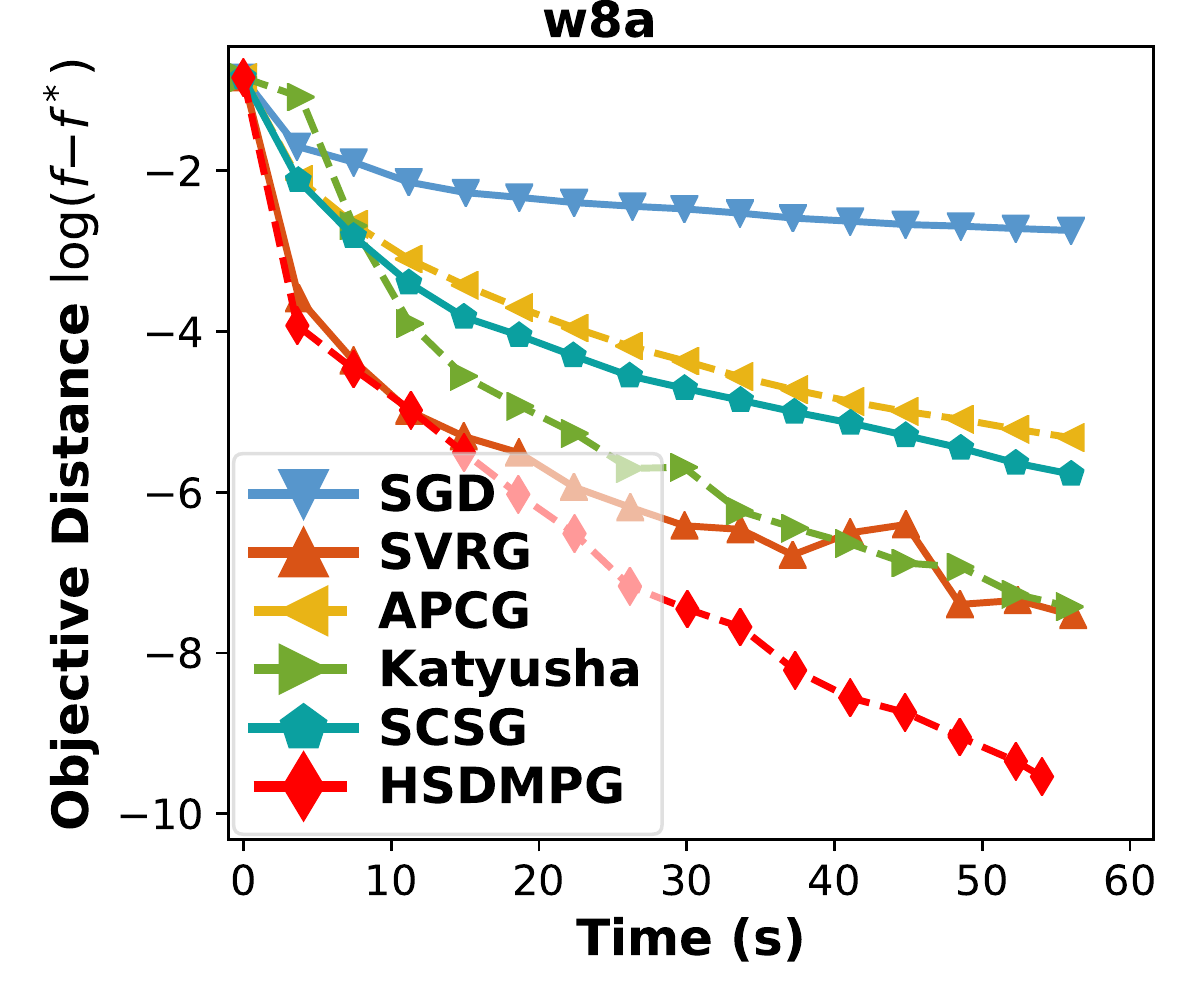}&
				\includegraphics[width=0.25\linewidth]{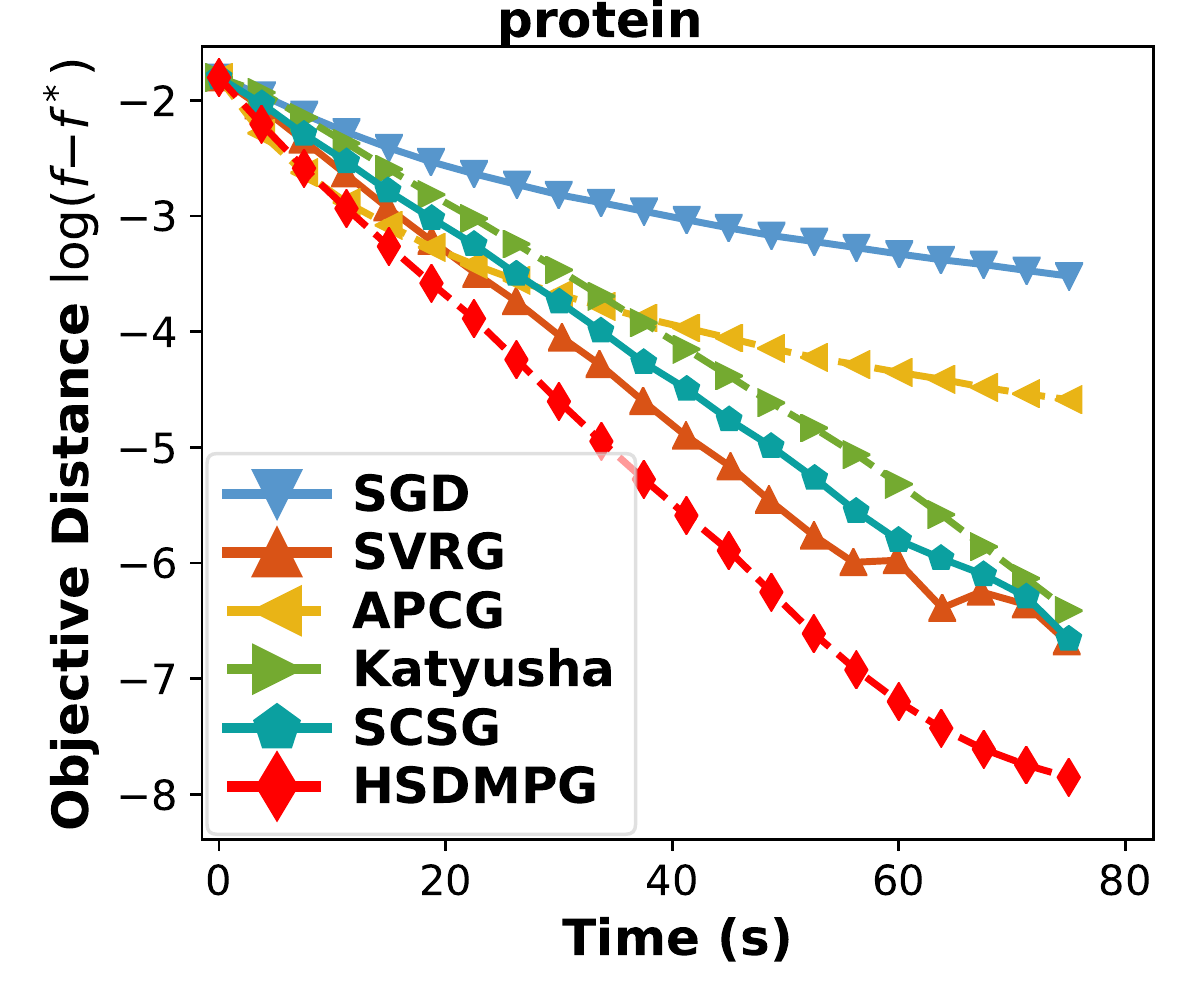}&
				\includegraphics[width=0.25\linewidth]{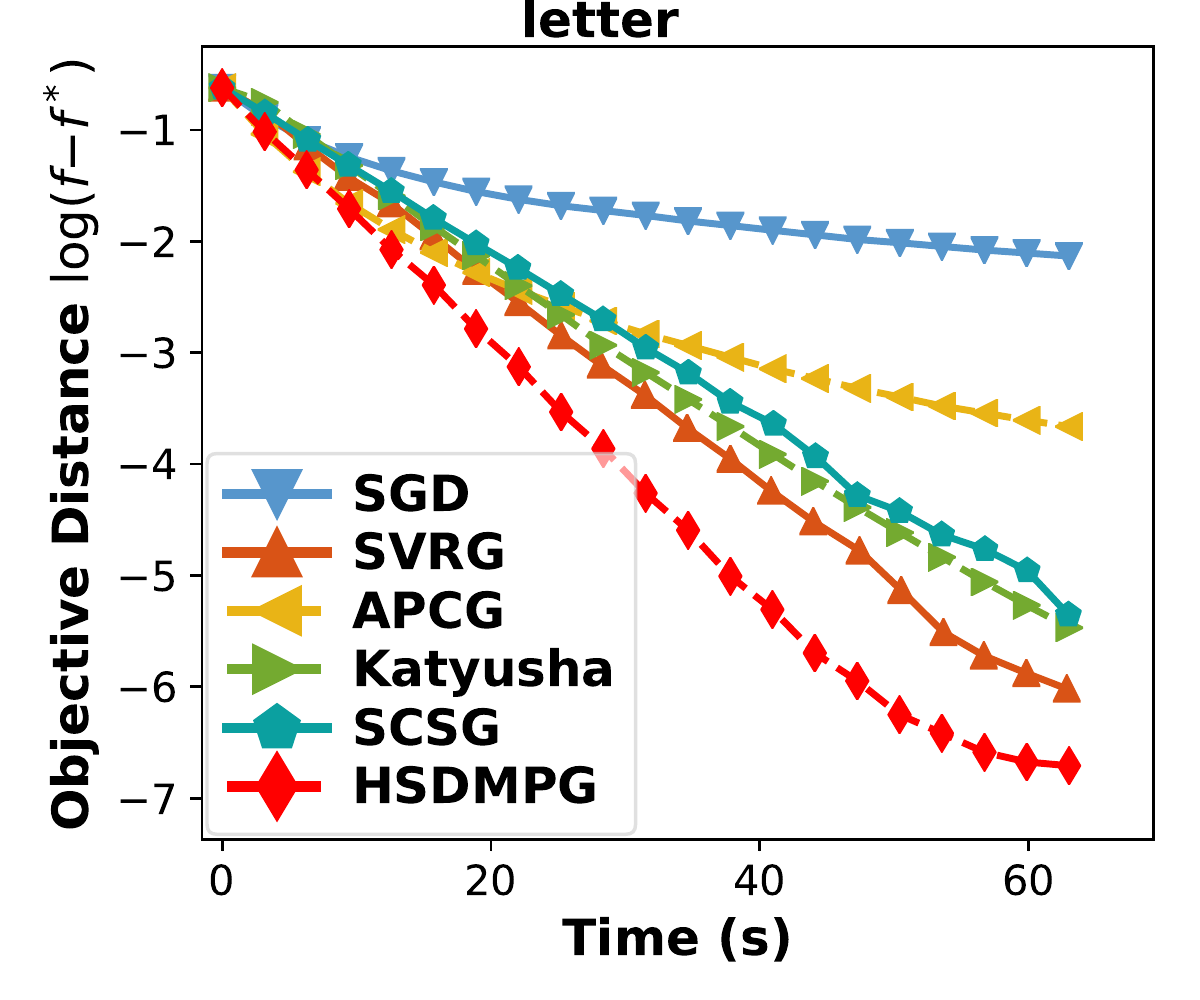}\\
			\end{tabular}
		\end{center}
		\vspace{-1.3em}
		\caption{Multi-epoch processing (about 8 epochs): stochastic gradient algorithms process data multiple pass on logistic regression  problems (\textsf{ijcnn} and \textsf{w08}) and softmax regression problems (\textsf{protein} and \textsf{letter}).
		}
		\label{illustration_convergence3}
		\vspace{0.4em}
	\end{figure*}
	
	Finally we consider a realistic case where the optimization error of problem~\eqref{eqn:general} matches the  intrinsic excess  error bound  $\mathcal{O}(1/\sqrt{n})$. For this case, as discussed at the end of Section~\ref{quatraticlosscomplexity} that the regularization parameter should  be set at the scale of $\mu=\mathcal{O}(1/\sqrt{n})$ with balanced impact against the guarantees on estimation error. As a result, the condition number $\kappa$ could scale as large as $\mathcal{O}(\sqrt{n})$. The following corollary substantializes the IFO complexity bound in Theorem~\ref{lemma:outer_loop_convergence} to such a setting. See Appendix~\ref{proofoptimizationerrorgeneralcase} for a proof of this result.
	
	\begin{cor}\label{optimizationerrorgeneralcase}
		Suppose   the assumptions  in Theorem~\ref{lemma:outer_loop_convergence}  hold. By setting $s\!=\!\mathcal{O}\big(\frac{\nu n^{0.75}\! \log^{0.5}\!(d)}{\log(n)}\big)$,  the IFO complexity of \HSDAN~on the generic loss to achieve $\mathbb{E}[F(\wmi{t}) \!-\!F(\wms) ] \!\le\! \frac{1}{\sqrt{n}}$ is of order   $\mathcal{O}\! \left(\nu^{0.5}n^{0.875} \!\log^{0.75}\!(d)\!\log^{2.25} \!\left(n\right) + \nu^2 n^{0.5} \right)\!.$
		\setbetters
	\end{cor}
	
	Corollary~\ref{optimizationerrorgeneralcase} shows that for generic convex loss, the IFO complexity of \HSDAN~to attain the $\mathcal{O} \big(1/\sqrt{n}\big)$ intrinsic excess error is of the order $\mathcal{O} \left(n^{0.875} \log^{2.25}\left(n\right) \right)$. This shows that \HSDAN~is able to achieve nearly optimal generalization with less than a single pass over data. Compared  with the complexity bound for the quadratic loss, such a more general IFO complexity bound of \HSDAN~only comes at the cost of a slightly increased overhead on the logarithmic factor, \ie, from $\log^{1.5}(n)$ for the quadratic case to the $\log^{2.25}(n)$ for generic convex loss. Similar to the observations in the quadratic case, from results in Table~\ref{comparisontable} one can observe that all the considered   state-of-the-art methods need to process the entire data at least one pass to achieve the desired optimization error for generic convex loss. All in all, the established theoretical results for both quadratic and non-quadratic loss functions showcase the benefit of \HSDAN~for efficient optimization of large-scale learning problems with near-optimal generalization.

	\section{Experiments}\label{experiments}
	
	In this section, we carry out experiments to compare the numerical performance of \HSDAN~with several representative stochastic gradient optimization algorithms, including  SGD~\cite{robbins1951stochastic}, SVRG~\cite{SVRG}, APCG~\cite{lin2014accelerated},  Katyusha~\cite{katyusha} and SCSG~\cite{lei2017less}.  We evaluate all the considered algorithms on two sets of strongly-convex learning tasks. The first set is for ridge regression with least squared loss $\ell(\wm^\top \xmi{i}, \ymi{i})=\frac{1}{2}\|\wm^\top \xmi{i} - \ymi{i}\|_2^2$, where $\ymi{i}$
	is the target output  of sample $\xmi{i}$. In the second setting we consider two classification models: logistic regression with  loss $\ell(\wm^\top \xmi{i}, \ymi{i})=\log\left(1+\exp(-\ymi{i}\wm^\top \xmi{i})\right)$  and multi-class softmax regression with  $k$-classification loss $\ell(\wm^\top \xmi{i}, \ymi{i})=\sum_{j=1}^k \bm{1}\{\ymi{i}=j\} \log\left(\frac{\exp(\wmi{j}^\top\xmi{i})}{\sum_{s=1}^k \exp(\wmi{s}^\top\xmi{i})}\right).$ We run simulations on ten datasets whose details are described in Appendix~\ref{append:more_experiment}.
	For \HSDAN, we set the size $s$ of $\Smmi{}$ around $n^{0.75}$. For the minibatch for inner problems, we  set initial  minibatch size $|\Smmi{1}|=50$ and then follow our theory to exponentially expand size of $\Smmi{t}$ with proper exponential rate. The regularization constant in the subproblem~\eqref{equat:P_t_w}  is set to be $\gamma\!=\!\sqrt{\log(d)/s}$ as suggested by our theory.   The optimization error $\vapi{t}$ in~\eqref{equat:P_t_w} is controlled by respectively allowing SVRG to run 3 epochs and 10 epochs on the two sets of tasks. Similarly, we control the optimization error  $\vapis{t}$  in~\eqref{eqn:epsilon} by running SVRG with 3 epochs. Since there is no ground truth on real data, we run FGD sufficiently long until $\|\nabla F(\tilde\wm)\|_2\!\leq\! 10^{-10}$  and take $F(\tilde\wm)$ as an approximate optimal value $F(\wms)$ for sub-optimality estimation.
	
	\subsection{Results for the quadratic loss}
	\textbf{Single-epoch  evaluation results.} Here we first evaluate well-conditioned quadratic problems such that moderately accurate solution can be obtained after only one epoch of data pass. Such a one epoch setting usually occurs in online learning. Towards this goal, we set the  regularization parameter $\mu=0.01$ to make the quadratic problems well-conditioned. From Figure~\ref{illustration_convergence}, one can observe that \HSDAN~exhibits much sharper convergence behavior than the considered baselines, though  most algorithms can achieve small optimization error after one epoch processing of data. This confirms the theoretical predictions in Corollaries~\ref{thrm:quadratic_hsdmpg_ifo} and~\ref{optimizationerrorcomplexity} that \HSDAN~is cheaper in IFO complexity than SGD and variance-reduced algorithms, \eg~SVRG and SCSG, when the data scale is large.

	\textbf{Multi-epoch evaluation results}. For more challenging problems, an algorithm usually requires  multiple cycles of data processing to achieve accurate optimization. Here we reset the regularization strength parameter in  quadratic problems as $\mu=10^{-4}$ for generating more challenging optimization tasks. As shown in Figure~\ref{illustration_convergence2}, one can again observe that \HSDAN~converges faster than all the compared algorithms in terms of IFO complexity. Particularly, we compare both IFO complexity and wall-clock running time on the \textsf{letter} and \textsf{rcv11} datasets. The convergence curves under these two   metrics consistently show the superior computational efficiency of \HSDAN~to the considered state-of-the-arts on large-scale learning tasks, which well support the theoretical predictions in Corollaries~\ref{thrm:quadratic_hsdmpg_ifo} and~\ref{optimizationerrorcomplexity}.
	
	\subsection{Results for the non-quadratic loss}\label{moreexp}
	Finally, we investigate the convergence performance of the proposed \HSDAN~on non-quadratic convex loss functions. Specifically, we evaluate all the compared algorithms on  logistic regression and its multi-classes version, \ie~softmax regression,  in which their regularization modulus parameters are set as $\mu=0.01$. Figure~\ref{illustration_convergence3} reports the running time evolving curves which can accurately reflects the efficiency of an algorithm. These results show that \HSDAN~converges significantly faster than the baseline algorithms for the considered non-quadratic loss functions, which well support the predictions in Theorem~\ref{lemma:outer_loop_convergence} and Corollary~\ref{optimizationerrorgeneralcase} that \HSDAN~has lower  IFO complexity than the state-of-the-arts in the regimes where data scale is large. This set of results also demonstrates the effectiveness of our sequential quadratic-approximation approach for extending the attractive computational complexity guarantees on quadratic loss to generic convex loss.

	\section{Conclusions}
	We proposed \HSDAN~as a hybrid stochastic-deterministic minibach proximal gradient method for $\ell_2$-regularized ERM problems. For quadratic loss, we showed that \HSDAN~enjoys provably lower computational complexity than prior state-of-the-art SVRG algorithms in large-scale settings. Particularly, to attain the optimization error $\epsilon\!=\!\mathcal{O}\big(1/\sqrt{n}\big)$ at the order of intrinsic excess error bound of ERM which is sufficient for generalization, the stochastic gradient complexity of \HSDAN~is dominated by $\mathcal{O} (n^{0.875})$ (up to logarithmic factors). To our best knowledge, \HSDAN~for the first time achieves nearly optimal generalization in less than a single pass over data. Almost identical computational complexity guarantees hold for an extension of \HSDAN~to generic strongly convex loss functions via sequential quadratic approximation. Extensive numerical results demonstrate the substantially improved computational efficiency of \HSDAN~over the prior methods. We expect that the algorithms and computational learning theory developed in this paper for $\ell_2$-regularized ERM can be extended to stochastic convex optimization problems. Also, it is worthwhile to explore the opportunity of using first-order acceleration techniques to further improve the computational complexity guarantees of \HSDAN.
	
	\section*{Acknowledgements}
	The authors sincerely thank the anonymous reviewers for their constructive comments on this work. Xiao-Tong Yuan is supported in part by National Major Project of China for New Generation of AI under Grant No.2018AAA0100400 and in part by Natural Science Foundation of China (NSFC) under Grant No.61876090 and No.61936005.
	
	\bibliographystyle{icml2020}
	\bibliography{referen}

	\appendix
\newpage
\onecolumn
\icmltitle{Hybrid Stochastic-Deterministic Minibatch Proximal Gradient: Less-Than-Single-Pass Optimization with Nearly Optimal Generalization\\
	(Supplementary File)}
\appendix
\begin{quote}	
	This supplementary document contains the technical proofs of convergence results and some additional numerical results of the paper entitled ``Hybrid Stochastic-Deterministic Minibatch Proximal Gradient: Less-Than-Single-Pass Optimization with Nearly Optimal Generalization''. It is structured as follows. Appendix~\ref{AuxiliaryLemmas} first present several auxiliary lemmas which will be used for subsequent analysis and whose proofs are deferred to Appendix~\ref{ProofforAuxiliaryLemmas}.  Then Appendix~\ref{append:dane_ls_analysis} gives the proofs of the main results in Sec.~\ref{resultsofquadratic}, including  Theorem~\ref{thrm:quadratic_hsdmpg}  which analyzes convergence rate of \HSDAN~and  Corollaries~\ref{thrm:quadratic_hsdmpg_ifo} and~\ref{optimizationerrorcomplexity}  which analyze the IFO complexity of \HSDAN~on the quadratic problems.  Next,  Appendix~\ref{proofof32} provides the proofs of the results in Sec.~\ref{proofofgeneralloss}, including Theorem~\ref{lemma:outer_loop_convergence} which proves the convergence rate of \HSDAN~and analyzes its IFO complexity for generic problems, and Corollary~\ref{optimizationerrorgeneralcase} which gives the IFO complexity of \HSDAN~to achieve the intrinsic excess error bound.  Then in Appendix~\ref{ProofforAuxiliaryLemmas} we present  the proofs of auxiliary   lemmas in Appendix~\ref{AuxiliaryLemmas}, including Lemmas~\ref{lemma:vector_concentration_bound2} $\sim$ \ref{lemma:precondition_bound}. Finally, more details of the testing datasets used in the manuscript are presented in  Appendix~\ref{append:more_experiment}.
\end{quote}

\section{Some Auxiliary Lemmas}\label{AuxiliaryLemmas}

Here we introduce auxiliary lemmas which will be used for proving the results in the manuscript. For the sake of readability, we defer the proofs of some lemmas into Appendix~\ref{ProofforAuxiliaryLemmas}. The following elementary lemma will be used frequently throughout our analysis.

\begin{lemma}\label{lemma:vector_concentration_bound2} 	Assume that the loss $F(\wm)$ is a $\mu$-strongly convex loss, $\sup_{\wm}\!\frac{1}{n}\!\sum_{i=1}^n $ $\|\Hm^{-1/2} (\nabla F(\wm) - \nabla \elli{i}(\wm)) \|_2^2 \le \nu^2$.
	Suppose $\rmi{t-1} = \nabla F(\wmi{t-1}) - \gmi{t-1}$ where $\gmi{t-1} =\nabla F_{\Smmi{t}}(\wmi{t-1})$. Then by setting
	\begin{equation*}
	|\Smmi{t}| = \frac{16\nu^2(\mu+2\gamma)^2}{\mu^2}\exp\left(\frac{\mu t}{\mu+2\gamma}\right)  \bigwedge n,
	\end{equation*}
	we have
	\begin{equation*}
	\mathbb{E}\left[\|\Hm^{-1/2}\rmi{t}\|^2\right] \!\le\! \frac{\mu^2}{16(\mu \!+\! 2\gamma)^2}\exp\left(\!-\frac{\mu t}{\mu \!+\! 2\gamma}\!\right),\ \   \ \mathbb{E}\left[\|\Hm^{-1/2} \rmi{t}\|\right]\! \le\!    \frac{\mu}{4(\mu \!+\! 2\gamma)}\exp\left(\!-\frac{\mu t}{2(\mu \!+\! 2\gamma)}\!\right)\!.
	\end{equation*}
\end{lemma}
See its proof in Appendix~\ref{prooflemma:vector_concentration_bound2}.

\begin{lemma}\label{expectationHessian}
	Suppose $\Hm$ and $\Hms$ respectively denote the Hessian matrix of $F(\wm)$ and $F_{\Smmi{}}(\wm)$ in problem~\eqref{eqn:general}.  w.l.o.g., suppose $\|\xmi{i}\|\leq \rx\ (i=1,\cdots,n)$ and $\ell(\wm^\top\xm,\ym)$ is $\Ls$-smooth w.r.t. $\wm^\top\xm$.  Then we have
	\begin{equation*}
	\begin{split}
	\EE_{\Smmi{}}\left[ \left\| \Hms -\Hm \right\|^2\right]  \leq \frac{(\sqrt{ \log(d)} +\sqrt{2} )^2 \Ls^2r^4}{s}\quad \text{and}\quad \EE_{\Smmi{}}\left[\|\Hms -\Hm\|\right] \le \frac{(\sqrt{ \log(d)} +\sqrt{2} )\Ls r^2}{\sqrt{s}},
	\end{split}
	\end{equation*}
	where $s$ is the size of $\Smmi{}$.
\end{lemma}
see its proof in Appendix~\ref{proofoflemma6}

\begin{lemma}\label{lemma:precondition_bound}
	Let $\Am$ and $\Bm$ be two symmetric and positive definite matrices and $\Bm\succeq \mu I$ for some $\mu>0$. If $\|\Am-\Bm\|\le \gamma$, then $(\Am+\gamma \Imm)^{-1}\Bm$ is diagonalizable and
	\begin{equation*}
	\frac{\mu}{\mu + 2\gamma} \leq \left\|\Bm^{1/2}(\Am+\gamma \Imm)^{-1}\Bm^{1/2}\right\| \leq 1.
	\end{equation*}
	Moreover, the following spectral norm bound holds:
	\[
	\|\Imm - \Bm^{1/2}(\Am+\gamma \Imm)^{-1} \Bm^{1/2}\| \le  \frac{2\gamma}{\mu + 2\gamma}.
	\]
\end{lemma}
See its proof in Appendix~\ref{proofoflemma1}.

\section{Proofs for the Results in Section~\ref{resultsofquadratic}}
\label{append:dane_ls_analysis}

We collect in this appendix section the technical proofs of the results in Section~\ref{resultsofquadratic} of the main paper.

\subsection{Proof of Theorem~\ref{thrm:quadratic_hsdmpg}}\label{proofoftheorem1}
\begin{proof}   This proof has four steps. To begin with, for brevity, let  $\umi{t} = \Hm^{1/2}(\wmi{t} - \wms)$.  In the first step, we  establish the relation between $\umi{t} $ and $\umi{t-1}$ which will be widely used for subsequent proof. Since for quadratic problems, we have $\EE[F(\wmi{t}) - F(\wms)] = \frac{1}{2}\EE[\|\wmi{t} - \wms\|^2_{\Hm}]$. So here we aim to upper bound $\EE[\|\wmi{t} - \wms\|^2_{\Hm}]$ first, and then use it to upper bound $\EE[F(\wmi{t}) - F(\wms)] $. To bound  the second-order moment $\EE[\|\wmi{t} - \wms\|^2_{\Hm}]$, we need to first bound its first-order moment  $\EE[\|\wmi{t} - \wms\|_{\Hm}]$.   So in the second step, we use the result in the first step to upper bound $\EE[\|\wmi{t} - \wms\|_{\Hm}]$. Then in the third step, we upper bound $\EE[\|\wmi{t} - \wms\|^2_{\Hm}]$. Finally, we can use above result to upper bound the loss.   Please see the proof steps below.

	\textbf{Step 1. Establish the relation between $\umi{t} $ and $\umi{t-1}$. }\\
	Since the objective function $F$ is quadratic, namely $F(\wm) = \frac{1}{2} (\wm-\wms)^T\Hm(\wm-\wms)$,  for any $\wmi{t-1}$ the optimal solution $\wms = \argmin_{\wm} F(\wm)$ can always be expressed as
	\begin{equation}\label{safcsadfcdas}
	\wms= \wmi{t-1} - \Hm^{-1} \nabla F(\wmi{t-1}).
	\end{equation}
	
	Then computing the gradient of $\Pmi{t-1}$  yields
	\[
	\begin{aligned}
	\nabla \Pmi{t-1}(\wmi{t})   =   \gmi{t-1}  + \nabla \Fms(\wmi{t}) - \nabla \Fms(\wmi{t-1}) + \gamma(\wmi{t} -\wmi{t-1} ),
	\end{aligned}
	\]
	where $\gmi{t-1} =\nabla F_{\Smmi{t}}(\wmi{t-1})$. 
	Let $\Hms$ denotes the Hessian matrix of the  loss on minibatch $\SSm$.	Considering   $\Hms(\wmi{t})\equiv \Hms$ holds in the quadratic case, we can obtain $\nabla \Fms(\wmi{t}) - \nabla \Fms(\wmi{t-1}) = \Hms(\wmi{t}-\wmi{t-1})$. Thus plugging this results into  $\nabla \Pmi{t-1}(\wmi{t}) $ further yields
	\[
	\begin{aligned}
	\wmi{t} =& \wmi{t-1} - (\Hms + \gamma \Imm)^{-1} \gmi{t-1} + (\Hms + \gamma \Imm)^{-1} \nabla \Pmi{t-1}(\wmi{t}) \\
	=& \wmi{t-1} - (\Hms + \gamma \Imm)^{-1} \nabla F(\wmi{t-1}) + (\Hms + \gamma \Imm)^{-1} \nabla \Pmi{t-1}(\wmi{t}) + (\Hms + \gamma \Imm)^{-1} \rmi{t-1},
	\end{aligned}
	\]
	where $\rmi{t-1} = \nabla F(\wmi{t-1}) - \gmi{t-1}$. Next plugging Eqn.~\eqref{safcsadfcdas} into the above equation, it establishes
	\begin{equation*}\label{equat:proof_quadratic_thrm_key_3}
	\wmi{t} - \wms = (\Imm - (\Hms + \gamma \Imm)^{-1} \Hm ) (\wmi{t-1} - \wms) + (\Hms + \gamma \Imm)^{-1} \nabla \Pmi{t-1}(\wmi{t})+ (\Hms + \gamma \Imm)^{-1} \rmi{t-1}.
	\end{equation*}
	By multiplying $\Hm^{1/2}$ on both sides of the above recurrent form we have
	\begin{equation*}
	\begin{split}
	\Hm^{1/2}(\wmi{t} - \wms)
	= &(\Imm \!-\! \Hm^{1/2}(\Hms \!+\! \gamma \Imm)^{-1} \Hm^{1/2})\Hm^{1/2}(\wmi{t-1} \!-\! \wms) \\ &+ \Hm^{1/2}(\Hms \!+\! \gamma \Imm)^{-1} \nabla \Pmi{t-1}(\wmi{t})
	\!+\! \Hm^{1/2}(\Hms \!+\! \gamma \Imm)^{-1} \rmi{t-1}.
	\end{split}
	\end{equation*}
	Since $\umi{t} = \Hm^{1/2}(\wmi{t} - \wms)$, we have
	\begin{equation}\label{aaaaafdsafdssd}
	\umi{t} =\! (\Imm \!-\! \Hm^{1/2}(\Hms +  \gamma \Imm)^{-1} \Hm^{1/2}) \umi{t} +\! \Hm^{1/2}(\Hms +\gamma \Imm)^{-1} \nabla \Pmi{t-1}(\wmi{t})
	\!+\! \Hm^{1/2}(\Hms + \gamma \Imm)^{-1} \rmi{t-1}.
	\end{equation}

	\textbf{Step 2. Upper bound $\EE[\|\umi{t}\|] $. }\\
	Conditioned on $\wmi{t-1}$ and based on the basic inequality $\|\bm{T}\xm\|\le \|\bm{T}\|\|\xm\|$ we get
	\begin{equation}\label{aaaaadat34tgewfdsaf}
	\begin{split}
	\mathbb{E}[\|\umi{t}\|]
	\le& \EE\left[\|\Imm \!-\! \Hm^{1/2}(\Hms \!+\! \gamma \Imm)^{-1} \Hm^{1/2}\| \|\umi{t-1}\| \!+\! \|\Hm^{1/2}(\Hms \!+\! \gamma \Imm)^{-1}\Hm^{1/2}\|\|\Hm^{-1/2}\nabla \Pmi{t-1}(\wmi{t})\|\right] \\
	&+ \EE\left[\|\Hm^{1/2}(\Hms + \gamma \Imm)^{-1}\Hm^{1/2}\|\EE[\|\Hm^{-1/2}\rmi{t-1}\|].\right]
	\end{split}
	\end{equation}
	
	From Lemma~\ref{lemma:vector_concentration_bound2}, we know that by setting $
	|\Smmi{t}| = \frac{16\nu^2(\mu+2\gamma)^2}{\mu^2}\exp\left(\frac{\mu t}{\mu+2\gamma}\right)  \bigwedge n$,
	then the inequality always holds
	\begin{equation*}\label{equationexpectationgeradient}
	\mathbb{E}\left[\|\Hm^{-1/2}\rmi{t}\|\right] \le \frac{\mu}{4(\mu + 2\gamma)}\exp\left(-\frac{\mu t}{2(\mu + 2\gamma)}\right).
	\end{equation*}
	Suppose $\|\xmi{i}\|\leq \rx\ (i=1,\cdots,n)$ and $\ell(\wm^\top\xm,\ym)$ is $\Ls$-smooth w.r.t. $\wm^\top\xm$.  Then by using Lemma~\ref{expectationHessian} we have
	\begin{equation*}
	\mathbb{E}\left[\|\Hms -\Hm\|\right] \le \gamma= \frac{(\sqrt{ \log(d)} +\sqrt{2} )\Ls r^2}{\sqrt{s}},
	\end{equation*}
	where $s$ is the size of $\Smmi{}$.  In this way, by using Lemma~\ref{lemma:precondition_bound}, we can further establish 
	\begin{equation}\label{EEAAFF}
	\frac{\mu}{\mu + 2\gamma} \leq\left\|\Hm^{1/2}(\Hms+\gamma \Imm)^{-1}\Hm^{1/2}\right\| \leq 1\quad \text{and} \quad \left\|\Imm - \Hm^{1/2}(\Hms+\gamma \Imm)^{-1} \Hm^{1/2}\right\| \le  \frac{2\gamma}{\mu + 2\gamma}.
	\end{equation}
	Similarly, we have $\|\Hm^{-1/2} \nabla \Pmi{t-1}(\wmi{t}) \| \leq \frac{1}{\sqrt{\mu}} \| \nabla \Pmi{t-1}(\wmi{t}) \|\leq \frac{\vapi{t}}{\sqrt{\mu}}$. Now we plug the above results into~Eqn.~\eqref{aaaaadat34tgewfdsaf} and establish
	\begin{equation*}
	\begin{split}
	\mathbb{E}[\|\umi{t}\|]
	\led{172}& \frac{2\gamma}{\mu + 2\gamma} \|\umi{t-1}\| + \frac{\varepsilon_t}{\sqrt{\mu}} + \mathbb{E}[\|\Hm^{-1/2}\rmi{t-1}\|] \\
	\led{173}& \left( 1- \frac{\mu}{\mu + 2\gamma}\right) \|\umi{t-1}\| + \frac{\mu}{4(\mu + 2\gamma)}\exp\left(-\frac{\mu(t-1)}{2(\mu + 2\gamma)}\right) + \frac{\mu}{4(\mu + 2\gamma)}\exp\left(-\frac{\mu(t-1)}{2(\mu + 2\gamma)}\right) \\
	=& \left( 1- \frac{\mu}{\mu + 2\gamma}\right) \|\umi{t-1}\| + \frac{\mu}{2(\mu + 2\gamma)}\exp\left(-\frac{\mu(t-1)}{2(\mu + 2\gamma)}\right),
	\end{split}
	\end{equation*}
	where in the inequality \ding{172} we have used   $\Hm\succeq \mu \Imm$,  \ding{173} follows from the condition $\varepsilon_t \le \frac{\mu^{1.5}}{4(\mu + 2\gamma)}\exp\left(-\frac{\mu(t-1)}{2(\mu + 2\gamma)}\right)$.

	By taking expectation with respect to $\wmi{t-1}$ we arrive at
	\[
	\mathbb{E}[\|\umi{t}\|] \le \left( 1- \frac{\mu}{\mu + 2\gamma}\right) \mathbb{E[}\|\umi{t-1}\|] + \frac{\mu}{2(\mu + 2\gamma)}\exp\left(-\frac{\mu(t-1)}{2(\mu + 2\gamma)}\right).
	\]
	By using induction and the basic fact $(1-a)\le \exp(-a), \forall a>0$ and for brevity let $a =  \frac{\mu}{2(\mu + 2\gamma)}$, the previous inequality then leads to
	\begin{equation*}
	\begin{split}
	\mathbb{E}[\|\wmi{t} - \wms\|_{\Hm}] = \mathbb{E}[\|\umi{t}\|] \le&  \left( 1- 2a\right) \mathbb{E[}\|\umi{t-1}\|] + a\exp\left(-a(t-1) \right)\\
	=&    \left( 1- 2a\right)^{t} \mathbb{E}[\|\umi{0}\|] + a\sum_{i=0}^{t-1}(1-2a)^{t-1-i} \exp\left(-ai\right) \\
	\leq &    \left( \frac{1- 2a}{1-a}\right)^{t} \mathbb{E}[\|\umi{0}\|] \exp(-at) + a\sum_{i=0}^{t-1} \left( \frac{1- 2a}{1-a}\right)^{t-1-i} \exp\left(-a(t-1)\right) \\
	\leq &    \left( \frac{1- 2a}{1-a}\right)^{t} \mathbb{E}[\|\umi{0}\|] \exp(-at) + (1-a)  \exp\left(-a(t-1)\right) \\
	\leq &   \left( \|\wmi{0} - \wmi*\|_{\Hm}+ (1-a) \exp(a) \right)  \exp\left(-at\right) \\
	\leq &   \left( \|\wmi{0} - \wmi*\|_{\Hm}+  \exp(2a) \right)  \exp\left(-at\right)  \\
	\leq &  \left( \|\wmi{0} - \wmi*\|_{\Hm}+  e \right)  \exp\left(-\frac{\mu t}{2(\mu + 2\gamma)}\right).
	\end{split}
	\end{equation*} 
	This means that for all $\umi{t}$, we have
	\begin{equation*}
	\begin{split}
	\mathbb{E}[\|\umi{t}\|] \leq  \left( \|\wmi{0} - \wmi*\|_{\Hm}+  e \right)  \exp\left(-\frac{\mu t}{2(\mu + 2\gamma)}\right).
	\end{split}
	\end{equation*}
	
	\textbf{Step 3. Upper bound $\EE[\|\umi{t}\|^2]$.}\\
	From Eqn.~\eqref{aaaaafdsafdssd}, we can upper bound $\EE[\|\umi{t}\|^2]$ as
	\begin{equation*}\label{aaaaafdsaf}
	\begin{split}
	\mathbb{E}[\|\umi{t}\|^2] = & \EE\left[\|(\Imm \!-\! \Hm^{1/2}(\Hms \!+\! \gamma \Imm)^{-1} \Hm^{1/2}) \umi{t-1}\|^2\! \right.\\
	& \left.+  \|\Hm^{1/2}(\Hms \!+\! \gamma \Imm)^{-1} \nabla \Pmi{t-1}(\wmi{t})\|^2 \!+\! \|\Hm^{1/2}(\Hms \!+\! \gamma \Imm)^{-1} \rmi{t-1}\|^2 \right] \\
	&+ 2\EE\left[ \langle (\Imm - \Hm^{1/2}(\Hms + \gamma \Imm)^{-1} \Hm^{1/2})\umi{t-1},  \Hm^{1/2}(\Hms + \gamma \Imm)^{-1} \nabla \Pmi{t-1}(\wmi{t}) \rangle\right]\\
	&+ 2\EE\left[ \langle (\Imm - \Hm^{1/2}(\Hms + \gamma \Imm)^{-1} \Hm^{1/2})\umi{t-1}, \Hm^{1/2}(\Hms + \gamma \Imm)^{-1} \rmi{t-1} \rangle\right]\\
	&+ 2\EE\left[ \langle \Hm^{1/2}(\Hms + \gamma \Imm)^{-1} \nabla \Pmi{t-1}(\wmi{t}), \Hm^{1/2}(\Hms + \gamma \Imm)^{-1} \rmi{t-1} \rangle\right].
	\end{split}
	\end{equation*}
	Since $ \EE_{\Smmi{t-1}}[\rmi{t-1}]=0$, it is easy to obtain
	\begin{equation*}\label{aaaaafdsaf}
	\begin{split}
	&\EE\left[ \langle (\Imm - \Hm^{1/2}(\Hms + \gamma \Imm)^{-1} \Hm^{1/2})\umi{t-1}, \Hm^{1/2}(\Hms + \gamma \Imm)^{-1} \rmi{t-1} \rangle\right]\\
	= & \EE_{\Smmi{}}\EE_{\Smmi{t-1}}\left[ \langle (\Imm - \Hm^{1/2}(\Hms + \gamma \Imm)^{-1} \Hm^{1/2})\umi{t-1}, \Hm^{1/2}(\Hms + \gamma \Imm)^{-1} \rmi{t-1} \rangle\right]\\
	= & \EE_{\Smmi{}}\left[ \langle (\Imm - \Hm^{1/2}(\Hms + \gamma \Imm)^{-1} \Hm^{1/2})\umi{t-1}, \Hm^{1/2}(\Hms + \gamma \Imm)^{-1} \EE_{\Smmi{t-1}}\rmi{t-1} \rangle\right]=0.
	\end{split}
	\end{equation*}

	Conditioned on $\wmi{t-1}$ and based on the basic inequality $\|\bm{T}\xm\|\le \|\bm{T}\|\|\xm\|$, we get
	\begin{equation}\label{aaaaafdsadsafcasf}
	\begin{split}
	&	\mathbb{E}[\|\umi{t}\|^2]\\
	\leq & \EE\left[\|(\Imm - \Hm^{1/2}(\Hms\! +\! \gamma \Imm)^{-1} \Hm^{1/2})  \| ^2 \|\umi{t-1}\|^2 \!+\! \|\Hm^{1/2}(\Hms \!+ \!\gamma \Imm)^{-1} \Hm^{1/2} \|^2 \|\Hm^{-1/2} \nabla \Pmi{t-1}(\wmi{t})\|^2 \right] \\
	&+ \EE\left[  \|\Hm^{1/2}(\Hms + \gamma \Imm)^{-1} \Hm^{1/2}\|^2 \|\Hm^{-1/2}\rmi{t-1}\|^2  \right]\\
	&+\! 2\EE\left[\| (\Imm \!-\! \Hm^{1/2}(\Hms \!+\! \gamma \Imm)^{-1} \Hm^{1/2})\|\!\cdot\! \|\umi{t-1}\| \! \cdot\! \|\Hm^{1/2}(\Hms \!+\! \gamma \Imm)^{-1}\Hm^{1/2}\|\!\cdot\! \|\Hm^{-1/2} \nabla \Pmi{t-1}(\wmi{t}) \|\right]\\
	&+ 2\EE\left[\| \Hm^{1/2}(\Hms + \gamma \Imm)^{-1} \Hm^{1/2}\|^2 \cdot \| \Hm^{-1/2} \nabla \Pmi{t-1}(\wmi{t}) \| \cdot \|  \Hm^{-1/2}  \rmi{t-1} \|\right].
	\end{split}
	\end{equation}
	
	From Lemma~\ref{lemma:vector_concentration_bound2}, we know that by setting $
	|\Smmi{t}| = \frac{16\nu^2(\mu+2\gamma)^2}{\mu^2}\exp\left(\frac{\mu t}{\mu+2\gamma}\right)  \bigwedge n$,
	then the inequality always holds
	\begin{equation*}\label{equationexpectationgeradient}
	\mathbb{E}\left[\|\Hm^{-1/2}\rmi{t}\|^2\right] \le \frac{\mu^2}{16(\mu + 2\gamma)^2}\exp\left(-\frac{\mu t}{\mu + 2\gamma}\right).
	\end{equation*}
	Suppose $\|\xmi{i}\|\leq \rx\ (i=1,\cdots,n)$ and $\ell(\wm^\top\xm,\ym)$ is $\Ls$-smooth w.r.t. $\wm^\top\xm$.  Then by using Lemma~\ref{expectationHessian} we have
	\begin{equation*}
	\mathbb{E}\left[\|\Hms -\Hm\|^2\right] \le \gamma^2= \frac{(\sqrt{ \log(d)} +\sqrt{2} )^2\Ls^2 r^4}{s},
	\end{equation*}
	where $s$ is the size of $\Smmi{}$.  In this way, by using Lemma~\ref{lemma:precondition_bound}, we can further establish
	\begin{equation*}
	\frac{\mu^2}{(\mu + 2\gamma)^2} \!\leq\!\left\|\Hm^{1/2}(\Hms+\gamma \Imm)^{-1}\Hm^{1/2}\right\|^2\!\leq\! 1\ \ \text{and}\  \  \left\|\Imm - \Hm^{1/2}(\Hms+\gamma \Imm)^{-1} \Hm^{1/2}\right\|^2 \!\le\!  \frac{4\gamma^2}{(\mu + 2\gamma)^2}.
	\end{equation*}
	Similarly, we have $\|\Hm^{-1/2} \nabla \Pmi{t-1}(\wmi{t}) \| \leq \frac{1}{\sqrt{\mu}} \| \nabla \Pmi{t-1}(\wmi{t}) \|\leq \frac{\vapi{t}}{\sqrt{\mu}}$. Now we plug the above results and Eqn.~\eqref{EEAAFF} into~Eqn.~\eqref{aaaaafdsadsafcasf} and establish
	\begin{equation*}
	\begin{split}
	\mathbb{E}[\|\umi{t}\|^2]
	\leq &\frac{4\gamma^2}{(\mu + 2\gamma)^2} \EE [\|\umi{t-1}\|^2] + \frac{\varepsilon_t^2}{\mu} +\frac{\mu^2}{16(\mu + 2\gamma)^2}\exp\left(-\frac{\mu t}{\mu + 2\gamma}\right) + \frac{8\gamma}{\mu + 2\gamma}  \frac{\varepsilon_t}{\sqrt{\mu}} \EE\left[  \|\umi{t-1}\| \right]\\
	&+ \frac{\vapi{t}}{\sqrt{\mu}}   \frac{\mu }{2(\mu + 2\gamma)}\exp\left(-\frac{\mu t}{2(\mu + 2\gamma)}\right).
	\end{split}
	\end{equation*}
	Finally, by using $
	\mathbb{E}[\|\umi{t}\|] \leq  \left( \|\wmi{0} - \wmi*\|_{\Hm}+  e \right)  \exp\left(-\frac{\mu t}{\mu + 2\gamma}\right)$ and $\varepsilon_t \le \frac{\mu^{1.5}}{4(\mu + 2\gamma)}\exp\left(-\frac{\mu(t-1)}{2(\mu + 2\gamma)}\right)$, we can obtain
	\begin{equation*}
	\begin{split}
	&\mathbb{E}[\|\umi{t}\|^2]\\
	\leq &\frac{4\gamma^2}{(\mu + 2\gamma)^2} \EE [\|\umi{t-1}\|^2]  \!+\!\frac{\mu^2}{8(\mu + 2\gamma)^2} \left( \frac{1}{2}\!\left(\!1\!+\!\exp\left(\frac{\mu}{\mu + 2\gamma}\right)\!\!\right) \!+\! \exp\left(\frac{\mu}{2(\mu + 2\gamma)}\right)\!\right)\!\exp\!\left(-\frac{\mu t}{\mu + 2\gamma}\right) \\
	&+\frac{2\mu\gamma b}{(\mu + 2\gamma)^2}  \exp\left(\frac{\mu}{2(\mu + 2\gamma)}\right) \exp\left(-\frac{\mu t}{\mu + 2\gamma}\right) \\
	\led{172} &\frac{4\gamma^2}{(\mu + 2\gamma)^2} \EE [\|\umi{t-1}\|^2]  +2 a^2 \exp\left(-2a t\right) +\frac{4b \gamma a ^2}{\mu}  \exp\left(-2at\right)\\
	= & \frac{4\gamma^2}{(\mu + 2\gamma)^2} \EE [\|\umi{t-1}\|^2]  +2 a^2 \left(1 +\frac{2b \gamma}{\mu}   \right) \exp\left(-2at\right) ,
	\end{split}
	\end{equation*}
	where $a =  \frac{\mu}{2(\mu + 2\gamma)}$ and  $b =\left( \|\wmi{0} - \wmi*\|_{\Hm}+  e \right)$. \ding{172} uses $\frac{1}{2}\left(1+\exp\left(\frac{\mu}{\mu + 2\gamma}\right)\right) + \exp\left(\frac{\mu}{2(\mu + 2\gamma)}\right)\leq 4$ and $ \exp\left(\frac{\mu}{2(\mu + 2\gamma)}\right) \leq 2$. By using induction and the basic fact $(1-a)\le \exp(-a), \forall a>0$ and for brevity letting $c=2 a^2 \left(1 +\frac{2b \gamma}{\mu}   \right)$, the previous inequality then leads to
	\begin{equation*}
	\begin{split}
	\mathbb{E}[\|\wmi{t} - \wms\|_{\Hm}^2] =  \mathbb{E}[\|\umi{t}\|^2] \le &  \left( 1- a^2\right) \mathbb{E[}\|\umi{t-1}\|^2] + c\exp\left(-2a t\right)\\
	=&    \left( 1- a^2\right)^{t} \mathbb{E}[\|\umi{0}\|^2] + c\sum_{i=1}^{t}(1-2a)^{t-i} \exp\left(-2ai\right) \\
	\leq &    \mathbb{E}[\|\umi{0}\|^2] \exp(-2at) + c  \exp\left(-2at\right) \\
	\leq &  \left( \|\wmi{0} - \wmi*\|_{\Hm}^2+  2 a^2 \left(1 +\frac{2b \gamma}{\mu}   \right) \right)  \exp\left(-\frac{\mu t}{\mu + 2\gamma}\right).
	\end{split}
	\end{equation*}
	
	\textbf{Step 4. Bound $\EE[F(\wmi{t}) - F(\wms)]$.}\\
	It is easy to check $\EE[F(\wmi{t}) - F(\wms)] = \frac{1}{2}\EE[\|\wmi{t} - \wms\|^2_{\Hm}]$ in the quadratic case. So we obtain the desired result:
	\begin{equation*}
	\begin{split}
	&\EE[F(\wmi{t}) \!- F(\wms)] \!=\! \frac{1}{2}\EE[\|\wmi{t} - \wms\|^2_{\Hm}]\\
	\leq  &  \frac{1}{2}\left( \|\wmi{0} - \wmi*\|_{\Hm}^2+  \frac{\mu^2}{2(\mu + 2\gamma)^2} \left(1 +\frac{2 \gamma}{\mu} \left( \|\wmi{0} - \wmi*\|_{\Hm}+  e \right)   \right) \right)  \exp\left(-\frac{\mu t}{\mu + 2\gamma}\right)\\
	\led{172}  &  \frac{1}{2}\left( \|\wmi{0} - \wmi*\|_{\Hm}^2+  \frac{1}{4} \|\wmi{0} - \wmi*\|_{\Hm} + \frac{3}{2} \right)  \exp\left(-\frac{\mu t}{\mu + 2\gamma}\right)\\
	= & \left(\frac{1}{2}\left( \|\wmi{0} \!-\! \wmi*\|_{\Hm} \!+\!  \frac{1}{2}\right)^2 \!+\! \frac{5}{8}\right) \exp\left(-\frac{\mu t}{\mu + 2\gamma}\right),
	\end{split}
	\end{equation*}
	where \ding{172} uses $ \frac{\mu^2}{2(\mu + 2\gamma)^2}  \leq \frac{1}{2}$ and  $ \frac{\mu \gamma}{(\mu + 2\gamma)^2}  \leq \frac{1}{4}$. The proof is completed.
\end{proof}

\subsection{Proof of Corollary~\ref{thrm:quadratic_hsdmpg_ifo}}\label{append:proof_of_dane_hb_ifo}

\begin{proof}
	This proof has four steps. In the first step, we estimate the smallest iteration number $T$ such that $\EE[F(\wmi{T}) - F(\wms)] \leq  \epsilon$. Since the IFO complexity comes from two aspects: (1) the outer sampling steps for constructing the proximal function  $\Pmi{t}(\wm) \!=\! \Fms(\wm) + \langle \nabla F_{\Smmi{t}}(\wmi{t-1}) \!-\! \nabla \Fms(\wmi{t-1}), \wm \rangle + \frac{\gamma}{2}\|\wm-\wmi{t-1}\|_2^2 $ which requires sampling the gradient $\nabla F_{\Smmi{t}}(\wmi{t-1})$; (2) the inner optimization complexity which is produced by SVRG to solve the inner problem $\Pmi{t}(\wm)$ such that $\|\Pmi{t}(\wm)\| \leq \vapi{t}$. So in the second step, we estimate computational complexity of the outer sampling. In the third step, we estimate computational complexity of the inner optimization via SVRG. Finally,  we combine these two kinds of complexity together  to obtain total IFO bounds. Please see the proof steps below.
	
	\textbf{Step 1. Estimate the smallest iteration number $T$ such that $\EE[F(\wmi{T}) - F(\wms)] \leq  \epsilon$.}  \\
	According to Theorem~\ref{thrm:quadratic_hsdmpg}, we have
	\begin{equation*}
	\begin{split}
	\mathbb{E}[F(\wmi{t})\!-\! F(\wms)]  \!=\!\frac{1}{2}\mathbb{E}[\|\wmi{t} \!- \!\wms\|_{\Hm}^2]
	\! \leq\! \zeta  \exp\Big(\!\!-\!\frac{\mu t}{\mu\!+\!2\gamma}\Big),
	\end{split}
	\end{equation*}
	where $\zeta\!=\!  \frac{1}{2}\left( \|\wmi{0} \!-\! \wmi*\|_{\Hm} \!+\!  \frac{1}{2}\right)^2 \!+\! \frac{5}{8} $ with $\|\wm\|_{\Hm}\!=\!\sqrt{\wm^\top\! \Hm \wm}$. In this way, to guarantee $\EE[F(\wmi{t}) - F(\wms)] \leq  \epsilon$, the iteration number $T$ should be satisfies
	\begin{equation*}
	\begin{split}
	T = \frac{ \mu+2\gamma}{\mu} \log \left(\frac{\zeta}{\epsilon}\right).
	\end{split}
	\end{equation*}
	\textbf{Step 2. Estimate computational complexity of the outer sampling .}  \\
	The stochastic gradient estimation complexity up to the time step $T$ is given by
	\[
	\begin{aligned}
	\sum_{t=0}^{T-1} |\Smmi{t}| \le & \frac{16\nu^2(\mu+2\gamma)^2}{\mu^2}\sum_{t=0}^{T-1} \exp\left( \frac{\mu t}{\mu+2\gamma}\right)
	= \frac{16\nu^2(\mu+2\gamma)^2}{\mu^2}\frac{\exp\left(\frac{\mu T}{\mu+2\gamma}\right)-1}{\exp\left(\frac{\mu}{\mu+2\gamma}\right)-1}\\
	\led{172}& \frac{16\nu^2(\mu+2\gamma)^2}{\mu^2} \frac{\mu+2\gamma}{2\mu} \frac{\zeta}{\epsilon} = \frac{16\zeta \nu^2(\mu+2\gamma)^3}{\mu^3\epsilon},
	\end{aligned}
	\]
	where in \ding{172} we have used the definition of $T$ such that $\exp\left(\frac{\mu T}{\mu+2\gamma}\right)=\frac{\zeta}{\epsilon}$ and the fact $\exp(a)\ge 1+ a, \forall a>0$. At the same time, we also have
	\[
	\sum_{t=0}^{T-1} |\Smmi{t}| \le nT =\frac{ (\mu+2\gamma)n}{\mu} \log \left(\frac{\zeta}{\epsilon}\right).
	\]
	By combing the above two inequalities we obtain the  computational complexity of the outer sampling as
	\[
	\frac{16\zeta \nu^2(\mu+2\gamma)^3}{\mu^3\epsilon} \bigwedge \frac{ (\mu+2\gamma)n}{\mu} \log \left(\frac{\zeta}{\epsilon}\right) = \mathcal{O}\left( \left(1+ \frac{\kappa^3 \log^{1.5}(d)}{s^{1.5}}\right)\frac{\nu^2}{\epsilon} \bigwedge\left(1+ \frac{\kappa  \log^{0.5}(d)}{s^{0.5}}\right) n \log \left(\frac{1}{\epsilon} \right) \right),
	\]
	where we use    $\gamma= \frac{(\sqrt{ \log(d)} +\sqrt{2} )\Ls r^2}{\sqrt{s}}$ and $\kappa=\frac{L}{\mu}$.
	
	\textbf{Step 3. Estimate computational complexity of the inner optimization via SVRG.}  \\
	At each iteration time stamp $t$, we need to optimize the inner problem $\Pmi{t}(\wm) \!=\! \Fms(\wm) + \langle \nabla F_{\Smmi{t}}(\wmi{t-1}) \!-\! \nabla \Fms(\wmi{t-1}), \wm \rangle + \frac{\gamma}{2}\|\wm-\wmi{t-1}\|_2^2  $. In $\Pmi{t}(\wm)$, its finites-sum structure comes from  $\Fms(\wm) $ and its gradient.
	
	For $(\mu+\gamma)$-strongly-convex and $(\Ls+\gamma)$-smooth problem,  it is standardly known that the IFO complexity of the inner-loop SVRG computation to achieve  $\EE[ \Pmi{t-1}(\wmi{T}) -  \Pmi{t-1}(\wms)]\le \vapi{t}$ can be bounded in expectation by $\mathcal{O}\left(\!\left(s + \frac{L+\gamma}{\gamma+\mu}\right)\!\log\! \left(\frac{1}{\epsilon_t}\right)\!\right)$, where $\wms$ denotes the optimal solution of $\Pmi{t-1}(\wm)$. Since $\Pmi{t-1}(\wm)$ is  $(\mu+\gamma)$-strongly-convex, we have $\|\nabla \Pmi{t-1}(\wmi{t})\|_2  \leq 2 (\mu+\gamma) (\Pmi{t-1}(\wmi{T}) -  \Pmi{t-1}(\wms))$. In this way, to achieve $\|\nabla \Pmi{t-1}(\wmi{t})\|_2  \leq\varepsilon_t = \frac{\mu^{1.5}}{4(\mu + 2\gamma)}\exp\left(-\frac{\mu(t-1)}{2(\mu + 2\gamma)}\right)$, the expected IFO complexity of SVRG is
	\[
	\begin{aligned}
	\mathcal{O}\left(\!\left(s + \frac{L+\gamma}{\gamma+\mu}\right)\!\log\! \left(\frac{2(\mu+\gamma)}{\epsilon_t}\right)\!\right) \leq & \mathcal{O}\left(\!\left(s + \frac{L}{\gamma}\right)\!\log\! \left(\frac{(\mu+\gamma)^2}{\mu^{1.5}}\exp\left(\frac{\mu(t-1)}{\mu+2\gamma}\!\right)\!\right)\!\right)\\
	=& \mathcal{O}\left(\!\left(s \!+\! \frac{L}{\gamma}\!\right)\!\left(\!\log\!\left(\frac{(\mu+\gamma)^2}{\mu^{1.5}}\right)+\frac{\mu(t-1)}{\mu+\gamma}\!\right)\!\right).
	\end{aligned}
	\]
	From above result  we know that $\mathbb{E}[F(w^{(t)})] \le F(w^*) + \epsilon$ after $T = \mathcal{O}\left(\frac{\gamma}{\mu} \log\left(\frac{1}{\epsilon}\right)\right)$ rounds of iteration. Therefore the total inner-loop IFO complexity is bounded in expectation by
	\[
	\begin{aligned}
	\mathcal{O}&\left(\!\sum_{t=1}^T\!\left\{\!\left(s + \frac{L}{\gamma}\right)\!\left(\log\left(\frac{(\mu+\gamma)^2}{\mu^{1.5}}\right)+\frac{\mu(t-1)}{\mu+\gamma}\right)\!\right\}\!\right) \!=\!\mathcal{O}\left(\!\left(s + \frac{L}{\gamma}\!\right)\!\!\left(T\log\!\left(\frac{(\mu+\gamma)^2}{\mu^{1.5}}\right) + \frac{\mu T^2}{\gamma}\right)\!\right) \\
	=& \mathcal{O}\left(\left(s+ \frac{L}{\gamma}\right)\left(\frac{\gamma}{\mu}\log\left(\frac{(\mu+\gamma)^2}{\mu^{1.5}}\right)\log\left(\frac{1}{\epsilon}\right) + \frac{\gamma}{\mu}\log^2\left(\frac{1}{\epsilon}\right)\right)\right).
	\end{aligned}
	\]
	We plug $\gamma= \frac{(\sqrt{ \log(d)} +\sqrt{2} )\Ls r^2}{\sqrt{s}}$ into the above inner-loop IFO bound to obtain
	\[
	\mathcal{O}\left(\left(s+ \sqrt{\frac{s}{\log(d)}}\right)\frac{L}{\mu}\sqrt{\frac{\log(d)}{s}}\left(\log\left(\frac{L^{1.5}}{\mu^{1.5}}\sqrt{\frac{\log(d)}{s}}\right)\log\left(\frac{1}{\epsilon}\right) + \log^2\left(\frac{1}{\epsilon}\right)\right)\right).
	\]
	
	\textbf{Step 4. Combing inner optimization complexity and outer sampling complexity to obtain total IFO bounds.}\\
	Combing the preceding inner-loop optimization complexity and outer sampling complexity  yields the following overall computation complexity bound
	\begin{equation*}
	\begin{split}
	&\mathcal{O} \left(\frac{L\sqrt{s\log(d)}}{\mu}\left(\log\left(\frac{L^{1.5}}{\mu^{1.5}}\sqrt{\frac{\log(d)}{s}}\right)\log\left(\frac{1}{\epsilon}\right) + \log^2\left(\frac{1}{\epsilon}\right)\right) + \left(1+ \frac{\kappa^3 \log^{1.5}(d)}{s^{1.5}}\right)\frac{\nu^2}{\epsilon} \bigwedge\left(1+ \frac{\kappa  \log^{0.5}(d)}{s^{0.5}}\right) n \log \left(\frac{1}{\epsilon} \right)  \right) \\
	=& \mathcal{O} \left( \kappa \sqrt{s\log(d)}   \log^2\left(\frac{1}{\epsilon}\right) +  \left(1+ \frac{\kappa^3 \log^{1.5}(d)}{s^{1.5}}\right)\frac{\nu^2}{\epsilon} \bigwedge\left(1+ \frac{\kappa  \log^{0.5}(d)}{s^{0.5}}\right) n \log \left(\frac{1}{\epsilon} \right)  \right),
	\end{split}
	\end{equation*}
	where $\kappa = \frac{\Ls}{\mu}$.
	
	This competes the proof.
\end{proof}

\subsection{Proof of Corollary~\ref{optimizationerrorcomplexity}}\label{proofofoptimizationerrorcomplexity}

\begin{proof}
	The result in Corollary~\ref{optimizationerrorcomplexity}  can be easily obtained. Specifically, we plug $\epsilon=\mathcal{O}(\frac{1}{\sqrt{n}})$ ,  $\kappa=\mathcal{O}(\sqrt{n})$ and $s=\mathcal{O}\big(\frac{\nu n^{0.75} \log^{0.5}(d)}{\log(n)}\big)$ into Corollary~\ref{thrm:quadratic_hsdmpg_ifo}  and can compute the desired results.
\end{proof}

\section{Proofs for the Results in Section~\ref{proofofgeneralloss}}\label{proofof32}

\subsection{Proof of Theorem~\ref{lemma:outer_loop_convergence}}\label{append:proof_of_generalloss_ifo}
\begin{proof} This proof has two steps. In the first step, we prove the results in the first part of Theorem~\ref{lemma:outer_loop_convergence}, namely the linearly convergence of $F(\wm)$ on the generic loss functions. Then in the second step, we analyze the computational complexity of \HSDAN~on the generic loss functions. Please see the following detailed steps.
	
	\textbf{Step 1. Establish linearly convergence of $F(\wm)$.}\\
	To begin with, by using the smoothness property of each individual loss function $\ell(\wm^\top\xm,\ym)$ we can obtain
	\begin{equation*}
	F(\wmi{t})  \leq \Qmi{t-1}(\wmi{t})\!= \!F(\wmi{t-1}) \!+\! \langle \nabla F(\wmi{t-1}), \wmi{t} \!-\! \wmi{t-1}\rangle \!+\! \Deltai{t-1}(\wmi{t}),
	\end{equation*}
	where $\Deltai{t-1}(\wm) = \frac{1}{2} (\wm - \wmi{t-1})^\top \Hmss (\wm - \wmi{t-1})$ with $\Hmss = \frac{\rhos}{n}\sum_{i=1}^n \xmi{i} \xmi{i}^\top + \mu \Imm$.
	
	On the other hand, from our optimization rule, we can establish for any $z \in [0,1]$
	\begin{equation*}
	\begin{split}
	&\Qmi{t-1}(\wmi{t}) \leq  \Qmi{t-1}( (1-z) \wmi{t} + z \wms) + \vapis{t}  \\
	=&  F(\wmi{t-1}) + z  \langle \nabla F(\wmi{t-1}), \wms -\wmi{t-1}\rangle  +  \frac{\rhos z^2}{2} (\wms - \wmi{t-1})^\top \!\left( \frac{1}{n}\sum_{i=1}^n\!\xmi{i} \xmi{i}^\top \!+ \!\frac{\mu}{\rhos} \Imm \right) (\wms - \wmi{t-1}) + \vapis{t}.
	\end{split}
	\end{equation*}
	Next, from the $\sigma$-strongly convexity of each loss $\ell(\wm^{\top}\xm,\ym)$, we can obtain $\nabla^2 F(\wm) =\frac{1}{n}\sum_{i=1}^n \ell''(\wm^\top \xmi{i}, \ymi{i}) \xmi{i} \xmi{i}^\top + \mu \Imm \succeq \frac{\sigma}{n}\sum_{i=1}^n \xmi{i} \xmi{i}^\top + \mu \Imm$ for all $\wm$. In this way, we can lower bound
	\begin{equation*}
	\begin{split}
	F(\wms) \geq& F(\wmi{t-1}) \!+\! \langle \nabla F(\wmi{t-1}), \wms \!-\! \wmi{t-1}\rangle +  \frac{\sigma}{2} (\wms - \wmi{t-1})^\top \left( \frac{1}{n}\sum_{i=1}^n \xmi{i} \xmi{i}^\top + \frac{\mu}{\sigma} \Imm \right) (\wms - \wmi{t-1})\\
	\ged{172}& F(\wmi{t-1}) \!+\! \langle \nabla F(\wmi{t-1}), \wms \!-\! \wmi{t-1}\rangle +  \frac{\sigma}{2} (\wms - \wmi{t-1})^\top \left( \frac{1}{n}\sum_{i=1}^n \xmi{i} \xmi{i}^\top + \frac{\mu}{\rhos} \Imm \right) (\wms - \wmi{t-1})\\
	\end{split}
	\end{equation*}
	where \ding{172} we use $\rhos\geq \sigma$. By setting $z=\frac{\sigma}{\rhos}$ and combining all results together, we have
	\begin{equation*}
	\begin{split}
	&F(\wmi{t})   \leq \Qmi{t-1}(\wmi{t})\\
	\leq &  F(\wmi{t-1}) + \frac{\sigma}{\rhos} \left[  \langle \nabla F(\wmi{t-1}), \wms -\wmi{t-1}\rangle  +  \frac{\sigma}{2} (\wms - \wmi{t-1})^\top \!\!\left( \frac{1}{n}\sum_{i=1}^n \!\xmi{i} \xmi{i}^\top \!+\! \frac{\mu}{\rhos} \Imm \right)\!\! (\wms - \wmi{t-1}) \right]\!+\! \vapis{t}\\
	\leq & F(\wmi{t-1}) + \frac{\sigma}{\rhos} \left[ F(\wms) - F(\wmi{t-1})  \right]+ \vapis{t}.
	\end{split}
	\end{equation*}
	Then by using the basic fact $(1-a)\le \exp(-a), \forall a>0$ and $\vapis{t}=\frac{\sigma}{2\rhos}\exp\left(- \frac{\sigma(t-1)}{2\rhos}\right)$ we rewrite this equation and obtain
	\begin{equation*}
	\begin{split}
	F(\wmi{t})   - F(\wms)  \leq & \left(1 - \frac{\sigma}{\rhos} \right) \left(  F(\wmi{t-1})   - F(\wms) \right) + \frac{\sigma}{2\rhos}\exp\left(- \frac{\sigma(t-1)}{2\rhos}\right)\\
	\lee{172} & \left(1 - 2a\right)^{t} \left(  F(\wmi{0})   - F(\wms) \right) +a\sum_{i=1}^{t} \left(1 - 2a \right)^{t-i}  \exp\left(- a(i-1) \right)\\
	\led{173} & \left(\frac{1 - 2a}{1-a}\right)^{t} \left(  F(\wmi{0})   - F(\wms) \right) \exp(-at) +a\sum_{i=1}^{t} \left(\frac{1 - 2a}{1-a} \right)^{t-i}  \exp\left(- a(t-1) \right)\\
	=& \left(\frac{1 - 2a}{1-a}\right)^{t} \left(  F(\wmi{0})   - F(\wms) \right) \exp(-at) +(1-a) \exp\left(- a(t-1) \right)\\
	\leq & \left(  F(\wmi{0})   - F(\wms)  + (1-a)\exp(a) \right) \exp(-at) \\
	\leq & \left(  F(\wmi{0})   - F(\wms)  + 1\right) \exp(-at),
	\end{split}
	\end{equation*}
	where in \ding{172} we let $a= \frac{\sigma}{2\rhos} $ for brevity; \ding{173} uses $(1-a)^{k} \leq \exp(-ak)$ for $a>0$.

	\textbf{Step 2. Establish computational complexity  of \HSDAN~for achieving $\EE[F(\wm)-F(\wms)]\leq \epsilon$.}\\
	It follows immediately that $\EE[F(\wm)-F(\wms)]\leq \epsilon$ is valid when
	\[
	t \ge \frac{2\rhos}{\sigma} \log \left(\frac{ F(\wmi{0}) - F(\wms)+1}{\epsilon}\right).
	\]

	At each iteration time stamp $t$, the leading terms in Theorem~\ref{thrm:quadratic_hsdmpg_ifo} suggest that the IFO complexity of the inner-loop \HSDAN~computation to achieve $\vapis{t}$-sub-optimality of $\Qmi{t}$ can be bounded in expectation by
	\begin{equation*}
	\begin{split}
	& \mathcal{O} \left( \kappa \sqrt{s\log(d)}   \log^2\left(\frac{1}{\vapis{t}}\right) +  \left(1+ \frac{\kappa^3 \log^{1.5}(d)}{s^{1.5}}\right)\frac{\nu^2}{\vapis{t}} \bigwedge\left(1+ \frac{\kappa  \log^{0.5}(d)}{s^{0.5}}\right) n \log \left(\frac{1}{\vapis{t}} \right)  \right)\\
	=&  \mathcal{O}\!\! \left(\!\!\frac{\sigma^2\sqrt{s\log(d)}}{\rhos\mu}t^2\!+\! \left(1+ \frac{\kappa^3 \log^{1.5}(d)}{s^{1.5}}\right)\!\frac{\rhos\nu^2}{\sigma}\!\exp\!\left(\frac{\sigma}{\rhos}t\right)\! \bigwedge\left(1+ \frac{\kappa  \log^{0.5}(d)}{s^{0.5}}\right) \frac{L n}{\sigma}  t \right)
	\end{split}
	\end{equation*}
	where $\kappa=\frac{\rhos}{\mu}$  denotes the conditional number and $\vapis{t}=\frac{\sigma}{2\rhos}\exp\left(- \frac{\sigma(t-1)}{2\rhos}\right)$.
	
	From above result, we know that $\EE[F(\wm)-F(\wms)]\leq \epsilon$ after $T = \mathcal{O}\left(\frac{\rhos}{\sigma} \log\left(\frac{1}{\epsilon}\right)\right)$ rounds of iteration. Therefore the total inner-loop IFO complexity (w.r.t. the quadratic sub-problem) is bounded in expectation by
	\[
	\begin{aligned}
	&\mathcal{O}\left(\sum_{t=1}^T\left\{  \frac{\sigma^2\sqrt{s\log(d)}}{\rhos\mu}t^2\!+\! \left(1+ \frac{\kappa^3 \log^{1.5}(d)}{s^{1.5}}\right)\!\frac{\rhos\nu^2}{\sigma}\!\exp\!\left(\frac{\sigma}{\rhos}t\right)\! \bigwedge\left(1+ \frac{\kappa  \log^{0.5}(d)}{s^{0.5}}\right) \frac{L n}{\sigma}  t  \right\}\right) \\
	=&\mathcal{O}\left(\frac{\sigma^2\sqrt{s\log(d)}}{\rhos\mu}T^3+\left(1+ \frac{\kappa^3 \log^{1.5}(d)}{s^{1.5}}\right)\frac{\rhos\nu^2}{\sigma}\exp\left(\frac{\sigma}{\rhos}(T+1)\right) \bigwedge\left(1+ \frac{\kappa  \log^{0.5}(d)}{s^{0.5}}\right) \frac{L n}{\sigma}  T^2  \right) \\
	=& \mathcal{O}\left(\frac{\rhos^2\sqrt{s\log(d)}}{\sigma\mu}\log^3\left(\frac{1}{\epsilon}\right) + \left(1+ \frac{\kappa^3 \log^{1.5}(d)}{s^{1.5}}\right) \frac{\rhos\nu^2}{\sigma \epsilon}  \bigwedge\left(1+ \frac{\kappa  \log^{0.5}(d)}{s^{0.5}}\right) \frac{L^3 n}{\sigma^3}  \log^2\left(\frac{1}{\epsilon}\right)   \right).
	\end{aligned}
	\]
	This proves the desired bound.
\end{proof}

\subsection{Proof of Corollary~\ref{optimizationerrorgeneralcase}}\label{proofoptimizationerrorgeneralcase}
\begin{proof}
	Based on Theorem~\ref{lemma:outer_loop_convergence}, the results  can be easily obtained. Specifically, we plug $\epsilon=\mathcal{O}(\frac{1}{\sqrt{n}})$ ,  $\kappa=\mathcal{O}(\sqrt{n})$ and  $s=\mathcal{O}\big(\frac{\nu n^{0.75} \log^{0.5}(d)}{\log(n)}\big)$ into Theorem~\ref{lemma:outer_loop_convergence} and can compute the desired results.
\end{proof}

\section{Proof of Auxiliary Lemmas}\label{ProofforAuxiliaryLemmas}

\subsection{Proof of Lemma~\ref{lemma:vector_concentration_bound2}}\label{prooflemma:vector_concentration_bound2}

The following lemma from~\cite{lei2017less} will be used to bound the gradient estimation variance.
\begin{lemma}\cite{lei2017less}\label{lemma:vector_concentration_bound}
	Let $z_1,...,z_N \in \mathbb{R}^p$ be an arbitrary population of $N$ vectors with $\sum_{i=1}^N z_i =0$. Let $S$ be a uniform random subset of $[N]$ with size $n$. Then
	\begin{equation*}
	\mathbb{E}\left\|\frac{1}{n}\sum_{i\in S} z_i\right\|^2 \le \frac{\mathbbm{1}(n<N)}{n} \frac{1}{N}\sum_{i=1}^N \|z_i\|^2.
	\end{equation*}
\end{lemma}
\begin{proof}[Proof of Lemma~\ref{lemma:vector_concentration_bound2}]
	Let $\zmii{t}{i} = \Hm^{-1/2} (\nabla F(\wmi{t}) - \nabla \elli{i}(\wm))$. Then we have $\sum_{i=1}^n \zmii{t}{i}  =0$, $\frac{1}{n}\sum_{i=1}^n\|\zmii{t}{i} \|^2\le\nu^2$ and $\Hm^{-1/2}\rmi{t}=\frac{1}{|\Smmi{t}|}\sum_{i\in \Smmi{t}} \zmii{t}{i}$. By invoking Lemma~\ref{lemma:vector_concentration_bound} we get
	\[
	\mathbb{E}\left[\|\Hm^{-1/2} \rmi{t}\|^2\right]= \mathbb{E}\left[\left\|\frac{1}{|\Smmi{t}|}\sum_{i\in \Smmi{t}}\zmii{t}{i} \right\|^2\right] \le \frac{\nu^2 \mathbbm{1}(|\Smmi{t}|<n)}{|\Smmi{t}|}.
	\]
	Provided that
	\[
	|\Smmi{t}| = \frac{16\nu^2(\mu+2\gamma)^2}{\mu^2}\exp\left(\frac{\mu t}{\mu+2\gamma}\right)  \bigwedge n,
	\]
	then the following condition always holds
	\begin{equation*}
	\mathbb{E}\left[\|\Hm^{-1/2}\rmi{t}\|^2\right] \le \frac{\mu^2}{16(\mu + 2\gamma)^2}\exp\left(-\frac{\mu t}{\mu + 2\gamma}\right).
	\end{equation*}
	Next, by using 	Jensen's Inequality, we can obtain
	\begin{equation*}
	\mathbb{E}\left[\|\Hm^{-1/2} \rmi{t}\|\right] \!\le\! \sqrt{\mathbb{E}\left[\|\Hm^{-1/2} \rmi{t}\|^2\right]} \!=\! \sqrt{\!\mathbb{E}\left[\left\|\frac{1}{|\Smmi{t}|}\sum_{i\in \Smmi{t}}\zmii{t}{i} \right\|^2\right]} \!\le\!  \frac{\mu}{4(\mu + 2\gamma)}\exp\left(-\frac{\mu t}{2(\mu + 2\gamma)}\right).
	\end{equation*}
	The proof is completed.
\end{proof}

\subsection{Proof of Lemma~\ref{expectationHessian}}\label{proofoflemma6}

\begin{lemma}\cite{oliveira2010sums}\label{expectationepsilon2}
	Suppose $\{\Ami{i}\}_{i=1}^n$ are deterministic Hermitian matrices and $\{\vapi{i}\}_{i=1}^{n}$ are independent Bernoulli variables taking values $\pm 1$ with probability $\frac{1}{2}$. Let $\Zm = \sum_{i=1}^{n} \vapi{i}\Ami{i}$. Then we have
	\begin{equation*}
	\begin{split}
	\EE_{\varepsilon}  \left[ \left\| \Zm \right\|^2 \right] \leq (\sqrt{ \log(d)} +\sqrt{2} )^2  \left\|\sum_{i=1}^n \Ami{i}^2 \right\|.
	\end{split}
	\end{equation*}
\end{lemma}

\begin{proof}
	To begin with, we can compute the Hessian matrix $\Hm =\frac{1}{n}\sum_{i=1}^n \ell''(\wm^\top \xmi{i}, \ymi{i}) \xmi{i} \xmi{i}^\top + \mu \Imm$.  In this way, we can formulate
	\begin{equation*}
	\begin{split}
	\left\|\Hms-\Hm\right\| = & \left\| \frac{1}{s}\sum_{i\in \Smmi{}} \ell''(\wm^\top \xmi{i}, \ymi{i}) \xmi{i} \xmi{i}^\top - \frac{1}{n}\sum_{i=1}^n \ell''(\wm^\top \xmi{i}, \ymi{i}) \xmi{i} \xmi{i}^\top \right\|.
	\end{split}
	\end{equation*}
	Assume  $\xmi{i}$ are drawn from  $  \Smmi{}$ and $\xmti{i}$ are drawn from $\Smmi{}'$ where  $\Smmi{}'$ is also uniformly sampled from the $n$ samples.  In this way, we can establish
	\begin{equation*}
	\begin{split}
	&\EE_{\Smmi{}}\left[ \left\| \frac{1}{s}\sum_{i\in \Smmi{}} \ell''(\wm^\top \xmi{i}, \ymi{i}) \xmi{i} \xmi{i}^\top - \frac{1}{n}\sum_{i=1}^n  \ell''(\wm^\top \xmi{i}, \ymi{i}) \xmi{i} \xmi{i}^\top \right\|^2\right]\\
	= &  \EE_{\Smmi{}}\left[ \left\| \frac{1}{s}\sum_{i=0}^{s}  \ell''(\wm^\top \xmi{i}, \ymi{i}) \xmi{i} \xmi{i}^\top - \EE_{\Smmi{}'} \frac{1}{s}\sum_{i=0}^{s}   \ell''(\wm^\top \xmti{i}, \ymti{i})  \xmti{i} \xmti{i}^\top \right\|^2\right] \\
	\led{172}&  \EE_{\Smmi{}}  \EE_{\Smmi{}'}\left[ \left\| \frac{1}{s}\sum_{i=0}^{s}  \ell''(\wm^\top \xmi{i}, \ymi{i}) \xmi{i} \xmi{i}^\top - \frac{1}{s}\sum_{i=0}^{s}   \ell''(\wm^\top \xmti{i}, \ymti{i})   \xmti{i} \xmti{i}^\top \right\|^2\right] \\
	\lee{173} &  \EE_{\varepsilon}\EE_{\Smmi{}}  \EE_{\Smmi{}'}\left[ \left\| \frac{1}{s}\sum_{i=1}^s  \varepsilon_i \left( \ell''(\wm^\top \xmi{i}, \ymi{i})\xmi{i} \xmi{i}^\top -   \ell''(\wm^\top \xmti{i}, \ymti{i})   \xmti{i} \xmti{i}^\top \right)\right\|^2\right] \\
	\leq & 4  \EE_{\varepsilon}\EE_{\Smmi{}} \left[ \left\| \frac{1}{s}\sum_{i=1}^s  \varepsilon_i  \ell''(\wm^\top \xmi{i}, \ymi{i}) \xmi{i} \xmi{i}^\top \right\|^2\right] \\
	\end{split}
	\end{equation*}
	where \ding{172} uses the Jensen's Inequality; in \ding{173}  the variable $\varepsilon$ has two values $ \pm 1$ with probability $\frac{1}{2}$. From Lemma~\ref{expectationepsilon2}, we have
	\begin{equation*}
	\begin{split}
	\EE_{\varepsilon}  \left[ \left\| \sum_{i=1}^s  \varepsilon_i  \ell''(\wm^\top \xmi{i}, \ymi{i}) \xmi{i} \xmi{i}^\top \right\|^2\right] \leq 	\Ls^2 \EE_{\varepsilon}  \left[ \left\| \sum_{i=1}^s  \varepsilon_i    \xmi{i} \xmi{i}^\top \right\|^2\right] \leq  (\sqrt{ \log(d)} +\sqrt{2} )^2 \Ls^2 \left\|\sum_{i=1}^s (\xmi{i} \xmi{i}^\top)^2 \right\|. 
	\end{split}
	\end{equation*}
	W.l.o.g., suppose $\|\xmi{i}\|\leq \rx$. Then we can obtain
	\begin{equation*}
	\begin{split}
	\EE_{\Smmi{}}\!\left[ \left\| \frac{1}{s}\sum_{i\in \Smmi{}} \!  \ell''(\wm^\top \xmi{i}, \ymi{i}) \xmi{i} \xmi{i}^\top \!-\! \frac{1}{n}\sum_{i=1}^n  \ell''(\wm^\top \xmti{i}, \ymti{i}) \xmi{i} \xmi{i}^\top \right\|^2\right] \!\leq & \frac{(\sqrt{ \log(d)} +\sqrt{2} )^2\Ls^2}{s} \EE_{\Smmi{}} \!\left\|\frac{1}{s}\sum_{i=1}^s \!(\xmi{i} \xmi{i}^\top)^2 \right\| \\
	\leq & \frac{(\sqrt{ \log(d)} +\sqrt{2} )^2 r^4\Ls^2}{s}.
	\end{split}
	\end{equation*}
	Therefore, we can further obtain
	\begin{equation*}
	\begin{split}
	\EE_{\Smmi{}}\left[ \left\| \Hms -\Hm \right\|^2\right]  \leq \frac{(\sqrt{ \log(d)} +\sqrt{2} )^2\Ls^2 r^4}{s}.
	\end{split}
	\end{equation*}
	Next, by using 	Jensen's Inequality, we can obtain
	\begin{equation*}
	\mathbb{E}\left[\|\Hms -\Hm\|\right] \le \sqrt{\mathbb{E}\left[\|\Hms -\Hm\|^2\right]} \leq \frac{(\sqrt{ \log(d)} +\sqrt{2} )\Ls r^2}{\sqrt{s}}.
	\end{equation*}
	The proof is completed.
\end{proof}

\subsection{Proof of Lemma~\ref{lemma:precondition_bound}}\label{proofoflemma1}
\begin{proof}
	Since both $\Am+\gamma \Imm$ and $\Bm$ are symmetric and positive definite, it is known that the eigenvalues of $(\Am+\gamma \Imm)^{-1}\Bm$ are positive real numbers and identical to those of $(A+\gamma I)^{-1/2}B(A+\gamma I)^{-1/2}$. Let us  consider the following eigenvalue decomposition of $(\Am+\gamma \Imm)^{-1/2}\Bm(\Am+\gamma \Imm)^{-1/2}$:
	\[
	(\Am+\gamma \Imm)^{-1/2}\Bm(\Am+\gamma \Imm)^{-1/2} = \Qm^\top \Lambda \Qm,
	\]
	where $\Qm^\top \Qm = \Imm$ and $\Lambda$ is a diagonal matrix with eigenvalues as diagonal entries. It is then implied that
	\[
	(\Am+\gamma \Imm)^{-1}\Bm= (\Am+\gamma \Imm)^{-1/2} \Qm^\top \Lambda \Qm(\Am+\gamma \Imm)^{1/2},
	\]
	which is a diagonal eigenvalue decomposition of $(\Am+\gamma \Imm)^{-1}\Bm $. Thus $(\Am+\gamma \Imm)^{-1}\Bm$ is diagonalizable.
	
	To prove the eigenvalue bounds of $(\Am+\gamma \Imm)^{-1}\Bm$, it suffices to prove the same bounds for $(\Am+\gamma \Imm)^{-1/2}\Bm(\Am+\gamma \Imm)^{-1/2}$. Since $\|\Am-\Bm\|\le \gamma$, we have $\Bm \preceq \Am + \gamma \Imm$ which implies $(\Am+\gamma \Imm)^{-1/2}\Bm(\Am+\gamma \Imm)^{-1/2}\preceq I$ and hence $\EE\left[\lambda_{\max}((\Am+\gamma \Imm)^{-1/2}\Bm(\Am+\gamma \Imm)^{-1/2})\right]\le 1$. Moreover, since $\Bm\succeq \mu \Imm$, it holds that $\frac{2\gamma}{\mu} \Bm - \gamma \Imm \succeq \gamma \Imm \succeq \EE_{\Am}\Am - \Bm$. Then we get $(\Am+\gamma \Imm)^{-1/2}\Bm(\Am+\gamma \Imm)^{-1/2}\succeq \frac{\mu}{\mu + 2\gamma}\Imm$ which implies $\lambda_{\min}((\Am+\gamma \Imm)^{-1/2}\Bm(\Am+\gamma \Imm)^{-1/2})\ge \frac{\mu}{\mu + 2\gamma}$. Similarly, we can show that $ \frac{\mu}{\mu + 2\gamma}I \preceq \Bm^{1/2}(\Am+\gamma \Imm)^{-1}\Bm^{1/2} \preceq I$, implying $\|\Imm - \Bm^{1/2}(\Am+\gamma \Imm)^{-1} \Bm^{1/2}\|\le  \frac{2\gamma}{\mu + 2\gamma}$. The proof is competed.
\end{proof}

\subsection{Descriptions of Testing Datasets}\label{append:more_experiment}
We first briefly introduce the ten testing datasets in the manuscript including   including \textsf{ijcnn}, \textsf{a09}, \textsf{w8a}, \textsf{covtype}, \textsf{protein}, \textsf{codrna}, \textsf{satimage}, \textsf{sensorless}, \textsf{letter}, \textsf{rcv1}. All these datasets are  provided in the LibSVM website\footnote[1]{https://www.csie.ntu.edu.tw/~cjlin/libsvmtools/datasets/}. Their detailed information is summarized in Table~\ref{Tabledatasets}. From it we can observe that these datasets are different from each other due to their feature dimension, training samples, and class numbers, \textit{etc}.

\begin{table}[h]
	\caption{Descriptions of the ten testing datasets.}
	\setlength{\tabcolsep}{6.5pt}  
	\renewcommand{\arraystretch}{0.94} 
	\label{Tabledatasets}
	\centering
	{ \footnotesize {
			\begin{tabular}{lccc|lccc}
				\toprule
				& $\#$class& $\#$sample& $\#$feature & & $\#$class& $\#$sample& $\#$feature\\  \midrule
				\textsf{ijcnn1} & 2 &49,990	&	22 & \textsf{codrna} & 2 &	59,535	&	8 \\
				\textsf{a09} & 2 & 32,561	&	123 & \textsf{satimage} & 6	&4,435      &    36 \\
				\textsf{w8a} &2 & 49,749 & 300   & \textsf{sensorless} & 11 &	58,509	&	48 \\
				\textsf{covtype} & 2 & 581,012	&54 & \textsf{rcv1}&  2& 20,242 &47,236 \\
				\textsf{protein}& 3	& 14,895   &     357& \textsf{letter}& 26 & 10,500 & 16  \\
				\bottomrule
				\hline
			\end{tabular}
	}}
\end{table}	
	
\end{document}